\documentclass[nohyperref]{article}

\usepackage{microtype}
\usepackage{graphicx}
\usepackage{subfigure}
\usepackage{booktabs} 
\usepackage{ytableau}
\usepackage{wrapfig}
\usepackage{multirow,makecell}

\usepackage{hyperref}
\usepackage{tikz}
\usepackage{enumerate}

\usepackage{xcolor}

\newcommand{\yl}[1]{{\color{magenta}{\bf[Yaron:} #1{\bf]}}}

\newcommand{\abs}[1]{\left\vert#1\right\vert}

\newcommand{\set}[1]{\left\{#1\right\}}
\newcommand{\parr}[1]{\left (#1\right )}
\newcommand{\brac}[1]{\left [#1\right ]}

\newcommand{\Real}{\mathbb R}
\newcommand{\Nat}{\mathbb N}

\newcommand{\eps}{\varepsilon}
\newcommand{\too}{\rightarrow}

\newcommand{\diag}{\textrm{diag}} 

\newcommand{\one}{\mathbf{1}}

\newcommand{\eg}{{e.g.}}
\newcommand{\ie}{{i.e.}}

\usepackage{amsmath,amsfonts,bm}

\def\eqref#1{equation~\ref{#1}}

\def\floor#1{\lfloor #1 \rfloor}
\def\1{\bm{1}}

\def\eps{{\epsilon}}

\def\vb{{\bm{b}}}

\def\vm{{\bm{m}}}

\def\vs{{\bm{s}}}

\def\vx{{\bm{x}}}

\def\mI{{\bm{I}}}

\DeclareMathAlphabet{\mathsfit}{\encodingdefault}{\sfdefault}{m}{sl}
\SetMathAlphabet{\mathsfit}{bold}{\encodingdefault}{\sfdefault}{bx}{n}
\newcommand{\tens}[1]{\bm{\mathsfit{#1}}}
\def\tA{{\tens{A}}}

\def\tH{{\tens{H}}}

\def\tR{{\tens{R}}}

\def\tX{{\tens{X}}}
\def\tY{{\tens{Y}}}
\def\tZ{{\tens{Z}}}

\def\gB{{\mathcal{B}}}

\def\gF{{\mathcal{F}}}

\def\gH{{\mathcal{H}}}

\def\gZ{{\mathcal{Z}}}

\def\sN{{\mathbb{N}}}

\DeclareMathOperator{\Tr}{Tr}

\usepackage{tabularx,ragged2e,booktabs,caption}
\newcolumntype{C}[1]{>{\Centering}m{#1}}

\newcolumntype{Z}[1]{>{\Left}m{#1}}

\definecolor{mygray}{gray}{.95}

\usepackage[accepted]{icml2023}

\usepackage{amsmath}
\usepackage{amssymb}
\usepackage{mathtools}
\usepackage{amsthm}

\usepackage[capitalize,noabbrev]{cleveref}

\theoremstyle{plain}
\newtheorem{theorem}{Theorem}[section]
\newtheorem{proposition}[theorem]{Proposition}
\newtheorem{lemma}[theorem]{Lemma}
\newtheorem{corollary}[theorem]{Corollary}
\theoremstyle{definition}
\newtheorem{definition}[theorem]{Definition}

\theoremstyle{remark}

\makeatletter
\newtheorem*{rep@theorem}{\rep@title}
\newcommand{\newreptheorem}[2]{%
\newenvironment{rep#1}[1]{%
 \def\rep@title{\textbf{#2} \ref{##1}}%
 \begin{rep@theorem}}%
 {\end{rep@theorem}}}
\makeatother

\newreptheorem{theorem}{Theorem}
\newreptheorem{proposition}{Proposition}
\newreptheorem{lemma}{Lemma}
\newreptheorem{corollary}{Corollary}

\newcommand{\Aut}{\mathrm{Aut}}
\newcommand{\scount}{\mathrm{count}}
\newcommand{\inj}{\mathrm{inj}}

\usepackage[textsize=tiny]{todonotes}

\icmltitlerunning{Equivariant Polynomials for Graph Neural Networks}

\begin{document}

\twocolumn[
\icmltitle{Equivariant Polynomials for Graph Neural Networks}



\icmlsetsymbol{equal}{*}

\begin{icmlauthorlist}
\icmlauthor{Omri  Puny}{equal,yyy}
\icmlauthor{Derek Lim}{equal,sch}
\icmlauthor{Bobak T. Kiani}{equal,sch}
\icmlauthor{Haggai Maron}{comp2}
\icmlauthor{Yaron Lipman}{yyy,comp}
\end{icmlauthorlist}

\icmlaffiliation{yyy}{Weizmann Institute of Science}
\icmlaffiliation{comp}{Meta AI Research Centre for Artificial Intelligence}
\icmlaffiliation{sch}{MIT CSAIL}

\icmlaffiliation{comp2}{NVIDIA Research}

\icmlcorrespondingauthor{Omri Puny}{omri.puny@weizmann.ac.il}

\icmlkeywords{Machine Learning, ICML}

\vskip .3in
]



\printAffiliationsAndNotice{\icmlEqualContribution} 

\begin{abstract}
Graph Neural Networks (GNN) are inherently limited in their expressive power. Recent seminal works  \cite{xu2019how,morris2019weisfeiler} introduced the Weisfeiler-Lehman (WL) hierarchy as a measure of expressive power. Although this hierarchy has propelled significant advances in GNN analysis and architecture developments, it suffers from several significant limitations. These include a complex definition that lacks direct guidance for model improvement and a WL hierarchy that is too coarse to study current GNNs.
This paper introduces an alternative expressive power hierarchy based on the ability of GNNs to calculate equivariant polynomials of a certain degree. As a first step, we provide a full characterization of all equivariant graph polynomials by introducing a concrete basis, significantly generalizing previous results. Each basis element corresponds to a specific multi-graph, and its computation over some graph data input corresponds to a tensor contraction problem. Second, we propose algorithmic tools for evaluating the expressiveness of GNNs using tensor contraction sequences, and calculate the expressive power of popular GNNs. Finally, we enhance the expressivity of common GNN architectures by adding polynomial features or additional operations / aggregations inspired by our theory. These enhanced GNNs demonstrate state-of-the-art results in experiments across multiple graph learning benchmarks. 
\end{abstract}

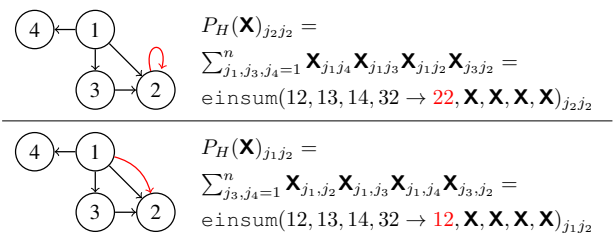
\begin{figure}
    \centering  
    \resizebox{\columnwidth}{!}{
     \begin{tabular}{cc}
 \makecell[c]{ \begin{tikzpicture}[main/.style = {draw, circle},baseline=5pt,semithick,] 
\node[main] (1) [minimum size=18pt] {3};
\node[main] (2) [minimum size=18pt, above of=1] {1};
\node[main] (3) [minimum size=18pt, right of=1] {2};
\node[main] (4) [minimum size=18pt, left of=2] {4};
\path 
(2) edge [->] node {} (1)
(4) edge [<-] node {} (2)
(2) edge [->] node {} (3)
(1) edge [->] node {} (3);
\path (3) edge [loop above, color=red] (3);
\end{tikzpicture}\\ \  }  & 
\makecell[bl]
{$P_H(\tX)_{j_2 j_2}=$\\ [5pt]
$\sum_{j_1,j_3,j_4=1}^n \tX_{j_1 j_4} \tX_{j_1 j_3 }\tX_{j_1 j_2}\tX_{j_3 j_2 }=$\\ [5pt]
     $\texttt{einsum}(12,13,14,32  \rightarrow {\color{red}  22} ,\tX,\tX,\tX,\tX)_{j_2j_2}$\vspace{-16pt} }
     \vspace{-7pt}\\ \hline \vspace{-7pt}\\
\makecell[c]{
\begin{tikzpicture}[main/.style = {draw, circle},baseline=5pt,semithick] 
\node[main] (1) [minimum size=18pt] {3};
\node[main] (2) [minimum size=18pt, above of=1] {1};
\node[main] (3) [minimum size=18pt, right of=1] {2};
\node[main] (4) [minimum size=18pt, left of=2] {4};
\path 
(2) edge [->, bend left, color=red] node {}  (3)
(2) edge [->] node {} (1)
(4) edge [<-] node {} (2)
(2) edge [->] node {} (3)
(1) edge [->] node {} (3);
\end{tikzpicture}   }  & 
\makecell[bl]{
$P_H(\tX)_{j_1 j_2}=$\\[5pt]
$\sum_{j_3,j_4=1}^n \tX_{j_1,j_2} \tX_{j_1,j_3}\tX_{j_1,j_4}\tX_{j_3,j_2}=$ \\ [5pt]
  $\texttt{einsum}(12,13,14,32 \rightarrow {\color{red}  12} ,\tX,\tX,\tX,\tX)_{j_1j_2}$ \vspace{-22pt} }
    \end{tabular}  }   
    \caption{Example of two basis elements of equivariant graph polynomials: node-valued (top), and edge-valued (bottom). Basis elements $P_H$ can be described by tensor contraction networks $H$ (left), corresponding to a matching \texttt{einsum} expression. }
    \label{fig:basis_examples}
\end{figure}

\section{Introduction}

In recent years, graph neural networks (GNNs) have become one of the most popular and extensively studied classes of machine learning models for processing graph-structured data. However, one of the most significant limitations of these architectures is their limited expressive power. In recent years,  the Weisfeiler-Lehman (WL) hierarchy has been used to measure the expressive power of GNNs \cite{morris2019weisfeiler,xu2019how,morris2021weisfeiler}. The introduction of the WL hierarchy marked an extremely significant step in the graph learning field, as researchers were able to evaluate and compare the expressive power of their architectures, and used higher-order WL tests to motivate the development of new, more powerful architectures. 

The WL hierarchy, however, is not an optimal choice for either purpose. First, its definition is rather complex and not intuitive, particularly for $k\geq 3$. One implication is that it is often difficult to analyze WL expressiveness of a particular architecture class. As a result, many models lack a theoretical understanding of their expressive power. A second implication is that WL does not provide practical guidance in the search for more expressive architecture. Lastly, as was noted in recent works \cite{morris2022speqnet,morris2019sparsewl}, the WL test appears to be too coarse to be used to evaluate the expressive power of current graph models. As an example, many architectures (e.g., \cite{frasca2022understanding}) are strictly more powerful than $2$-WL and bounded by $3$-WL, and there is no clear way to compare them. 

The goal of this paper is to offer an alternative expressive power hierarchy, which we call \emph{polynomial expressiveness} that mitigates the limitations of the WL hierarchy. Our proposed hierarchy relies on the concept of graph polynomials, which are, for graphs with $n$ nodes, polynomial functions $P:\mathbb{R}^{n^2}\to \mathbb{R}^{n^2}$ that are also permutation equivariant --- that is, well defined on graph data.
The polynomial expressiveness hierarchy is based on a natural and simple idea --- the ability of GNNs to compute or approximate equivariant graph polynomials up to a certain degree. 

This paper provides a number of theoretical and algorithmic contributions aimed at defining the polynomial hierarchy, providing tools to analyze the polynomial expressive power of GNNs, and demonstrating how this analysis can suggest practical improvements in existing models that give state-of-the-art performance in GNN benchmarks.

First, while some polynomial functions were used in GNNs in the past \cite{maron2019provably, chen2019equivalence, azizian2021expressive}, a complete characterization of the space of polynomials is lacking. In this paper, we provide the first characterization of graph polynomials with arbitrary degrees. In particular, we propose a basis for this vector space of polynomials, where each basis polynomial $P_H$ of degree $d$ corresponds to a specific multi-graph $H$ with $d$ edges. This characterization provides a significant generalization of known results, such as the basis of constant and linear equivariant functions on graphs \cite{maron2018invariant}.
Furthermore, this graphical representation $H$ can be viewed as a type of a tensor network, which provides a concrete way to compute those polynomials by performing a series of tensor (node) contractions. This is illustrated in Figure \ref{fig:basis_examples}.

As a second contribution, we propose tools for measuring polynomial expressiveness of graph models and placing them in the hierarchy. This is accomplished by analyzing tensor networks using standard contraction operators, similar to those found in Einstein summation (\texttt{einsum}) algorithms. Using these, we analyze two popular graph models: Message Passing Neural Networks (MPNNs) and Provably Powerful Graph Networks (PPGNs). This is done by first studying the polynomial expressive power of prototypical versions of these algorithms, which we define.

Our third contribution demonstrates how to improve MPNN and PPGN by using the polynomial hierarchy. Specifically, we identify polynomial basis elements that are not computable by existing graph architectures and add those polynomial basis elements to the model as feature layers. Also, we add two simple operations to the PPGN architecture (matrix transpose and diagonal / off-diagonal MLPs) to achieve the power of a Prototypical edge-based graph model. After precomputing the polynomial features, we achieve strictly better than $3$-WL expressive power while only requiring $O(n^2)$ memory --- to the best of our knowledge this is the first equivariant model to achieve this. We demonstrate that these additions result in state-of-the-art performance across a wide variety of datasets.

\section{Equivariant Graph Polynomials}
\label{s:graph_equi_poly}
We represent a graph with $n$ nodes as a matrix $\tX\in \Real^{n^2}$, where edge values are stored at off-diagonal entries, $\tX_{ij}$, $i\ne j$, $i,j \in [n]=\set{1,2,\ldots,n}$, and node values are stored at diagonal entries $\tX_{ii}$, $i\in [n]$. 

An \emph{equivariant graph polynomial} is a matrix polynomial map $P: \Real^{n^2} \too \Real^{n^2}$ that is also equivariant to node permutations. More precisely, $P$ is a polynomial map if each of its entries, $P(\tX)_{ij}$, $i,j\in [n]$, is a polynomial in the inputs $\tX_{rs}$, $r,s\in[n]$. $P$ is equivariant if it satisfies 
\begin{equation}\label{e:equivariance}
 P(g\cdot \tX) = g \cdot P(\tX),
\end{equation}
for all permutations $g\in S_n$, where $S_n$ denotes the permutation group on $[n]$, and $g$ acts on a matrix $\tY$ as usual by
\begin{equation}\label{e:g}
 (g\cdot \tY)_{ij} = \tY_{g^{-1}(i),g^{-1}(j)}.
\end{equation}


\subsection{$P_H$: Basis for equivariant graph polynomials}\label{ss:P_H}
We next provide a full characterization of equivariant graph polynomials by enumerating a particular basis, denoted $P_H$. In later sections we use this basis to analyze expressive properties of graph models and improve expressiveness of existing GNNs.

The basis elements $P_H$ of degree $d$ equivariant graph polynomials are enumerated from non-isomorphic \emph{directed multigraphs}, $H=(V,E,(a,b))$, where $V=[m]$ is the node set; $E=\set{(r_1,s_1),\ldots,(r_d,s_d)}$, $r_i,s_i\in [m]$,  the edge set, where parallel edges and self-loops are allowed; and $a,b\in V$ is a pair of not necessarily distinct nodes representing the output dimension. The pair $(a,b)$ will be marked in our graphical notation as a red edge. 


Defining the basis $P_H$ will be facilitated by the use of Einstein summation operator defined next. Note that the multi-graph $H$ can be represented as the following string that encodes both its list of edges and the single red edge: $H\cong \ 'r_1 s_1, \ldots, r_ds_d \too ab'$. The $\texttt{einsum}$ operator is:
\begin{align*}
    &\texttt{einsum}(H,\tX^1,\ldots,\tX^d)_{i_a,i_b} = \\  &\texttt{einsum}(r_1s_1,\ldots,r_ds_d\too ab,\tX^1,\ldots,\tX^d)_{i_a,i_b}= \\ &\sum_{\substack{j_1,\ldots j_m \in [n] \\ j_a=i_a, j_b=i_b}} \tX^1_{j_{r_1},j_{s_1}}\cdots \tX^d_{j_{r_d},j_{s_d}}
\end{align*}
\begin{figure}

    \centering  
    \resizebox{0.8\columnwidth}{!}{
     \begin{tabular}{cc}
 \makecell[c]{ \begin{tikzpicture}[main/.style = {draw, circle},baseline=5pt,semithick,] 
\node[main] (1) [minimum size=4pt] {\scalebox{.7}{3}};
\node[main] (2) [minimum size=4pt, above of=1] {\scalebox{.7}{1}};
\node[main] (3) [minimum size=4pt, right of=1] {\scalebox{.7}{2}};
\path 
(2) edge [->, bend left, color=red] node {}  (3)
(2) edge [->] node {\vspace{-5pt}$\tX$} (1)
(1) edge [->] node {$\tY$} (3);
\end{tikzpicture}\\ \  }  & 
\makecell[bl]{\scalebox{.9}{$\texttt{einsum}(H,\tX,\tY)_{j_1 j_2}=$}
     \\ [5pt]
     \scalebox{.9}{$\texttt{einsum}(13,32  \rightarrow {\color{red}  12} ,\tX,\tY)=$}
     \\ [5pt]
     \scalebox{.9}{$\sum_{j_3=1}^n \tX_{j_1 j_3 }\tY_{j_3 j_2 }= (\tX\tY)_{j_1j_2}$}\vspace{-15pt}}
    \end{tabular}  }   
    \vspace{-10pt}
    \caption{Example of matrix multiplication, $\tX\tY$. Computation defined by a multigraph $H$ and Einstein summation.
    }
    \label{fig:matrix_multiplication}
\end{figure}
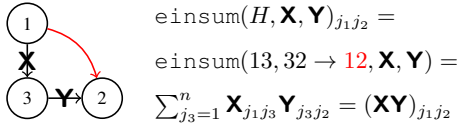
Figure \ref{fig:matrix_multiplication} shows how matrix multiplication can be defined using a corresponding multigraph $H$ and Einstein summation. Such multigraphs span the basis of equivariant polynomials:
\newpage
%
\begin{theorem}[Graph equivariant basis]\label{thm:poly}
A basis for all equivariant graph polynomials $P:\Real^{n^2} \too \Real^{n^2}$ of degree $\leq d$ is enumerated by directed multigraphs $H=(V,E,(a,b))$, where $|V|\leq \min\set{n,2+2d}$, $\abs{E}\leq d$, and $V\setminus \set{a,b}$ does not contain isolated nodes. The polynomial basis elements corresponding to $H$ are
\begin{equation}\label{eq:einsum_P_h}
    P_H(\tX) = \texttt{einsum}(H,\overset{d\text{ times}}{\overbrace{\tX,\ldots,\tX}}).
\end{equation}
\end{theorem}
An explicit formula for $P_H$ can be achieved by plugging in the definition of $\texttt{einsum}$ and \eqref{eq:einsum_P_h}:
\begin{align}\label{e:P_H}
    P_H(\tX)_{i_{a},i_b} &= \sum_{\substack{j_1,\ldots j_m \in [n]\\ j_a=i_a, j_b=i_b}}\prod_{(r,s)\in E} \tX_{j_r,j_s}.
\end{align}
Figure \ref{fig:basis_examples} depicts an example of graph equivariant basis elements $P_H$ corresponding to two particular multigraphs $H$. Note that a repeated pair $(a,a)$ in $H$ leads to a \emph{node-valued} equivariant polynomial, while a distinct pair $(a,b)$, $a\ne b$ leads to an \emph{edge-valued} equivariant polynomial. Furthermore, we make the convention that if $E$ is empty then $\prod_{(r,s)\in E}\tX_{i_r,i_s}=1$. The number of such polynomials increases exponentially with the degree of the polynomial; the first few counts of degree $d$ equivariant graph polynomials are 2 ($d=0$), 15 ($d=1$), 117 ($d=2$), 877 ($d=3$), 6719 ($d=4$), \dots Further details and proofs of these sequences are provided in \Cref{app:counting_section}. The full proof of Theorem \ref{thm:poly} is provided in \Cref{a:proof_of_poly}.

\textit{Proof idea for Theorem \ref{thm:poly}.} Since the set of monomials form a basis of all (not necessarily invariant) polynomials $P:\Real^{n^2}\too \Real^{n^2}$, we can project them onto the space of equivariant polynomials via the symmetrization (Reynolds) operator to form a basis for the equivariant polynomials. This projection operation will group the monomials into orbits that form the equivariant basis. 

To find these orbits, the basic idea is to consider the monomials in the input variables $\{\tX_{i,j}: i,j \in [n]\}$ and an additional variable $\{\delta_{i,j}: i,j \in [n]\}$ to denote the possible output entries of the equivariant map. Any given monomial $M(\tX, \delta^{a,b} )$ takes the form
\begin{equation}
    M(\tX, \delta^{a,b} )_{i,j} = \delta^{a,b}_{i,j} \prod_{r,s=1}^n \tX_{r,s}^{\tA_{r,s}},
\end{equation}
where $\tA\in \Nat_0^{n^2}$, $\Nat_0=\set{0,1,\ldots}$ is the matrix of powers, and  $\delta^{a,b}_{i,j}=1$ if $a=i$, $b=j$, and $\delta^{a,b}_{i,j}=0$ otherwise. A natural way to encode these monomials is with \emph{labeled} multi-graphs $H=(V,E,(a,b))$, where $V=[n]$, $E$ is defined by the adjacency matrix $\tA$, and $(a,b)$ is a special (red) edge. We therefore denote $M=M_H$.

These monomials can be projected onto equivariant polynomials via the Reynolds operator that takes the form,
\begin{equation}\label{e:H_orbits}
\begin{split}
    Q_H(\tX) &= \sum_{g\in S_n} g\cdot M_H(g^{-1}\cdot \tX, \delta^{g(a),g(b)}) \\
    &= \sum_{g\in S_n} M_{g\cdot H} (\tX, \delta^{a,b} ),
\end{split}
\end{equation}
where the action of $g\in S_n$ on the multi-graph $H$ is defined (rather naturally) as node relabeling of $H$ using the permutation $g$. 

From the above, we note: (i) $Q_H$ sums all monomials with multi-graphs in the orbit of $H$, namely $[H]=\set{g\cdot H\vert g\in S_n}$. This shows that, in contrast to $M_H$, $Q_H$ is represented by an \emph{unlabeled} multi-graph $H$ and enumerated by \emph{non-isomorphic} multi-graphs $H$. (ii) Since the symmetrization is a projection operator, any equivariant polynomial is spanned by $Q_H$. (iii) Since each $Q_H$ is a sum of $M_H$ belonging to a different orbit, and since orbits are disjoint, the set $\set{Q_H}$ for non-isomorphic $H$ is linearly independent. These three points establish that $\set{Q_H}$ for non-isomorphic multi-graphs $H$ is a basis of equivariant graph polynomials. 

Noting that $Q_H(\tX)_{ij}$ includes only terms for which $\delta^{g(a),g(b)}=\delta^{i,j}$, the explicit form below can be derived:
\begin{equation}\label{e:Q_H}
    Q_H(\tX)_{i_{a},i_b} = \sum_{\substack{j_1 \neq \ldots \neq  j_m \in [n]\\ j_a=i_a, j_b=i_b}}\prod_{(r,s)\in E} \tX_{j_r,j_s}.
\end{equation} $Q_H$ is similar to $P_H$ in \eqref{e:P_H}, except we only sum over non-repeated indices. The proof in \Cref{a:proof_of_poly} shows that $P_H$ is also a basis for such equivariant polynomials. 

\begin{figure}[t]
  \centering
    \resizebox{\columnwidth}{!}{
    \newcolumntype{?}{!{\vrule width 3pt}}
    \begin{tabular}{@{}c@{\hspace{2pt}}?c@{\hspace{2pt}}|@{\hspace{2pt}}c@{\hspace{2pt}}|@{\hspace{2pt}}c@{\hspace{2pt}}|@{\hspace{2pt}}c@{\hspace{2pt}}|@{\hspace{2pt}}c@{}}
\begin{tikzpicture}[main/.style = {draw, circle},baseline=0pt,semithick] 
\node[main] (1) [above of=1] {};
\path (1) edge [loop above, color=red] (1);
\end{tikzpicture}  & 
\begin{tikzpicture}[main/.style = {draw, circle},baseline=0pt,semithick] 
\node[main] (1) [] {};
\node[main] (2) [above of=1] {};
\draw[<-] (1) edge (2);
\path (2) edge [loop above, color=red] (2);
\end{tikzpicture} & 
\begin{tikzpicture}[main/.style = {draw, circle},baseline=0pt,semithick] 
\node[main] (1) [] {};
\node[main] (2) [above of=1] {};
\draw[<-] (2) edge (1);
\path (2) edge [loop above, color=red] (2);
\end{tikzpicture} & 
\begin{tikzpicture}[main/.style = {draw, circle},baseline=0pt,semithick] 
\node[main] (1) [] {};
\node[main] (2) [above of=1] {};
\path (1) edge [loop above] (1);
\path (2) edge [loop above, color=red] (2);
\end{tikzpicture} & 
\begin{tikzpicture}[main/.style = {draw, circle},baseline=0pt,semithick] 
\node[main] (2) [above of=1] {};
\path (2) edge [loop above] (2);
\path (2) edge [loop above, color=red] (2);
\path (2) edge [loop below, color=black] (2);
\end{tikzpicture} &
\begin{tikzpicture}[main/.style = {draw, circle},baseline=0pt,semithick] 
\node[main] (1) [] {};
\node[main] (2) [above of=1] {};
\node[main] (3) [right of=1]{};
\draw[->] (1) edge (3);
\path (2) edge [loop above, color=red] (2);
\end{tikzpicture}   \\  \hline 
\begin{tikzpicture}[main/.style = {draw, circle},baseline=0pt,semithick] 
\node[main] (1) [] {};
\node[main] (3) [right of=1] {};
\draw[->] (1) edge [color=red] (3);
\end{tikzpicture} &
\begin{tikzpicture}[main/.style = {draw, circle},baseline=0pt,semithick] 
\node[main] (1) [] {};
\node[main] (3) [right of=1] {};
\draw[->,bend right=30] (1) edge (3);
\draw[->,bend left=30] (1) edge [color = red] (3);
\end{tikzpicture} &
\begin{tikzpicture}[main/.style = {draw, circle},baseline=0pt,semithick] 
\node[main] (1) [] {};
\node[main] (3) [right of=1] {};
\draw[<-,bend right=30] (1) edge (3);
\draw[->,bend left=30] (1) edge [color = red] (3);
\end{tikzpicture} &
\begin{tikzpicture}[main/.style = {draw, circle},baseline=0pt,semithick] 
\node[main] (1) [] {};
\node[main] (3) [right of=1] {};
\path (1) edge [loop above] (1);
\draw[->] (1) edge [color = red] (3);
\end{tikzpicture} & 
\begin{tikzpicture}[main/.style = {draw, circle},baseline=0pt,semithick] 
\node[main] (1) [] {};
\node[main] (3) [right of=1] {};
\path (3) edge [loop above] (3);
\draw[->] (1) edge [color = red] (3);
\end{tikzpicture} & 
\begin{tikzpicture}[main/.style = {draw, circle},baseline=0pt,semithick] 
\node[main] (1) [] {};
\node[main] (3) [right of=1]{};
\node[main] (2) [below of=1]{};
\path (2) edge [loop above] (2);
\draw[->] (1) edge [color = red] (3); 
\end{tikzpicture} \\  \hline 
&
\begin{tikzpicture}[main/.style = {draw, circle},baseline=0pt,semithick] 
\node[main] (1) [] {};
\node[main] (2) [right of=1]{};
\node[main] (3) [below of=1]{};
\draw[->] (1) edge (3);
\draw[->] (1) edge [color = red] (2);
\end{tikzpicture} &
\begin{tikzpicture}[main/.style = {draw, circle},baseline=0pt,semithick] 
\node[main] (1) [] {};
\node[main] (2) [right of=1]{};
\node[main] (3) [below of=1]{};
\draw[<-] (1) edge (3);
\draw[->] (1) edge [color = red] (2);
\end{tikzpicture} &
\begin{tikzpicture}[main/.style = {draw, circle},baseline=0pt,semithick] 
\node[main] (1) [] {};
\node[main] (2) [right of=1]{};
\node[main] (3) [below of=1]{};
\draw[<-] (3) edge (2);
\draw[->] (1) edge [color = red] (2);
\end{tikzpicture} &
\begin{tikzpicture}[main/.style = {draw, circle},baseline=0pt,semithick] 
\node[main] (1) [] {};
\node[main] (2) [right of=1]{};
\node[main] (3) [below of=1]{};
\draw[->] (3) edge (2);
\draw[->] (1) edge [color = red] (2);
\end{tikzpicture} &
\begin{tikzpicture}[main/.style = {draw, circle},baseline=0pt,semithick] 
\node[main] (1) [] {};
\node[main] (2) [right of=1]{};
\node[main] (3) [below of=1]{};
\node[main] (4) [right of=3]{};
\draw[->] (3) edge (4);
\draw[->] (1) edge [color = red] (2); 
\end{tikzpicture} \\ 
\end{tabular}}
\caption{Basis of equivariant constant (left of bold line) and linear (right of bold line) graph polynomials.}\label{fig:linear_basis}
\end{figure}
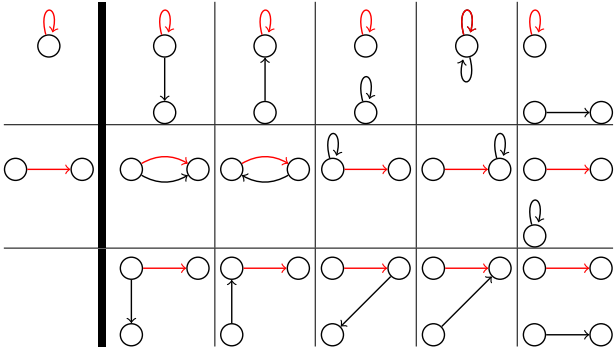
%

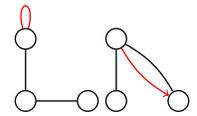
\begin{wrapfigure}[7]{r}{.3\columnwidth}
  \begin{center}\vspace{-20pt}\hspace{-10pt}
  \resizebox{.30\columnwidth}{!}{
  \begin{tabular}{@{\hskip0pt}c@{\hskip3pt}c@{\hskip0pt}}    
    \begin{tikzpicture}[main/.style = {draw, circle},baseline=5pt,semithick, every loop/.style={}] 
\node[main] (1) [] {};
\node[main] (2) [above of=1] {};
\node[main] (3) [right of=1] {};
\draw[-] (1) edge (2);
\draw[-] (3) edge (1);
\path (2) edge [loop above, color=red] (2);
\end{tikzpicture} 
        & 
\begin{tikzpicture}[main/.style = {draw, circle},baseline=5pt,semithick] 
\node[main] (1) [] {};
\node[main] (2) [above of=1] {};
\node[main] (3) [right of=1] {};
\draw[-] (1) edge (2);
\draw[->, color=red, bend right=15] (2) edge (3);
\draw[-, bend left=15] (2) edge (3);
\end{tikzpicture} \end{tabular} }
    \end{center}\vspace{-10pt}
  \caption{Simple $H$ corresponding to simple graph data.}\label{fig:symmetric_basis_examples}
\end{wrapfigure}
\textbf{Simple graphs.} It is often the case that the input data $\tX$ is restricted to some subdomain of $\Real^{n^2}$, e.g., 
symmetric $0/1$ matrices with diagonal entries set to zero correspond to simple graph data.
In such cases, polynomials $P_H$ that correspond to 
different multi-graphs $H$ can coincide, resulting in a smaller basis. 
For simple graph data $\tX$, existence of self loops in $H$ would result in $P_H(\tX)=0$, parallel edges in $H$ can be replaced with single edges without changing the value of $P_H(\tX)$, and since the direction of black edges in $H$ do not change the value of $P_H(\tX)$ we can consider only \emph{undirected} multi-graphs $H$. That is, for simple graph data it is enough to consider simple graphs $H$ (ignoring the red edge). Figure \ref{fig:symmetric_basis_examples} shows two examples of $H$ for simple graph data. 

\textbf{Example: linear basis.} \ Employing Theorem \ref{thm:poly} for the $d=0,1$ case reproduces the graph equivariant constant and linear functions from \citet{maron2018invariant}. Figure \ref{fig:linear_basis} depicts the graphical enumeration of the 2 constant and 15 linear basis elements. 

\begin{figure*}[t]
    \centering  
    \resizebox{\textwidth}{!}{
    \begin{tabular}{@{\hskip0pt}c@{\hskip0pt}c@{\hskip0pt}c@{\hskip0pt}c@{\hskip0pt}c@{\hskip0pt}c@{\hskip0pt}c@{\hskip0pt}}
\begin{tikzpicture}[main/.style = {draw, circle},baseline=5pt,semithick,every loop/.style={}]
\node[main] (1) [] {{\tiny $3$}};
\node[main] (2) [above of=1] {{\tiny $1$}};
\node[main] (3) [right of=1] {{\tiny $2$}};
\node[main] (4) [left of=2] {{\tiny $4$}};
\path 
(2) edge [->, bend left, color=red] node {}  (3)
(2) edge [->] node {$\tX$} (1)
(4) edge [<-] node {$\tX$} (2)
(2) edge [->] node {$\tX$} (3)
(1) edge [->] node {$\tX$} (3);
\end{tikzpicture}  
&
\makecell[b]{\begin{tikzpicture}[main/.style = {draw, circle},baseline=5pt,semithick,every loop/.style={}]
\node[main] (1) [] {{\tiny $1$}};
\node[main] (2) [left of=1,fill=lightgray] {{\tiny $4$}};
\path 
(1) edge [->] node {}  (2)
(1) edge [-, loop above, color=red] node {} (1);
\end{tikzpicture} \quad  $\tY_{j_1j_1}=\displaystyle\sum_{j_4}\tX_{j_1j_4}$      \\ \vspace{+0pt}
     $ \xrightarrow{\hspace*{140pt}} $
}

&
\begin{tikzpicture}[main/.style = {draw, circle},baseline=5pt,semithick,every loop/.style={}]
\node[main] (1) [] {{\tiny $3$}};
\node[main] (2) [above of=1] {{\tiny $1$}};
\node[main] (3) [right of=1] {{\tiny $2$}};
\path 
(2) edge [->, bend left, color=red] node {}  (3)
(2) edge [->] node {$\tX$} (1)
(2) edge [->] node {$\tX$} (3)
(2) edge [-,loop above] node {$\tY$} (2)
(1) edge [->] node {$\tX$} (3);
\end{tikzpicture} 

&

\makecell[b]{\begin{tikzpicture}[main/.style = {draw, circle},baseline=5pt,semithick,every loop/.style={}]
\node[main] (1) [] {{\tiny $2$}};
\node[main] (2) [left of=1,fill=lightgray] {{\tiny $3$}};
\node[main] (3) [above of=2] {{\tiny $1$}};
\path 
(3) edge [->] node {}  (2)
(2) edge [->] node {}  (1)
(3) edge [->, color=red] node {} (1); 
\end{tikzpicture}   $\tZ_{j_1j_2}=\displaystyle\sum_{j_3}\tX_{j_1j_3}\tX_{j_3j_2}$     \\ \vspace{+0pt}
     $ \xrightarrow{\hspace*{150pt}} $
}

&

\begin{tikzpicture}[main/.style = {draw, circle},baseline=5pt,semithick,every loop/.style={}]
\node[main] (3) [] {{\tiny $2$}};
\node[main] (1) [left of=3, draw=none] {};
\node[main] (2) [above of=1] {{\tiny $1$}};
\path 
(2) edge [->, bend left, color=red] node {}  (3)
(2) edge [->, bend right] node {$\tZ$}  (3)
(2) edge [->] node {$\tX$} (3)
(2) edge [-,loop above] node {$\tY$} (2);
\end{tikzpicture}  

&

\makecell[b]{\begin{tikzpicture}[main/.style = {draw, circle},baseline=5pt,semithick,every loop/.style={}]
\node[main] (3) [] {{\tiny $2$}};
\node[main] (2) [above left of=3] {{\tiny $1$}};
\path 
(2) edge [->, bend left, color=red] node {}  (3)
(2) edge [->, bend right] node {}  (3)
(2) edge [->] node {} (3)
(2) edge [-,loop above] node {} (2);
\end{tikzpicture}  $\tR_{j_1j_2}=\tY_{j_1j_1}\tZ_{j_1j_2}\tX_{j_1j_2}$      \\ \vspace{+0pt}
     $ \xrightarrow{\hspace*{140pt}} $
}

&

\begin{tikzpicture}[main/.style = {draw, circle},baseline=5pt,semithick,every loop/.style={}]
\node[main] (3) [] {{\tiny $2$}};
\node[main] (2) [above left of=3] {{\tiny $1$}};
\path 
(2) edge [->, bend left, color=red] node {}  (3)
(2) edge [->] node {\hspace{-10pt}\vspace{5pt}$\tR$} (3);
\end{tikzpicture}  
\end{tabular} }   
    \caption{Computation of $P_H(\tX)$ with a sequence of tensor contractions: The polynomial $P_H(\tX)$ (left-most) is computed when a single black edge parallel to the red edge is left (right-most); above each arrow is the tensor contraction applied (contracted nodes are in gray). \vspace{-5pt}
    }
    \label{fig:computation_of_P_H}
\end{figure*}
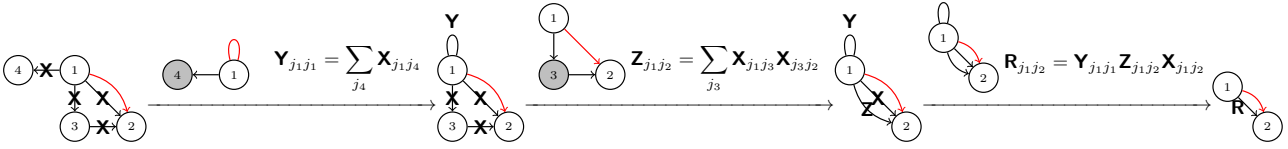

\textbf{Computing $P_H(\tX)$ with tensor contractions.}
A useful observation for the graph model analysis performed later is that computing $P_H(\tX)$ is equivalent to a tensor contraction guided by $H$. Similarly to $\texttt{einsum}$, computing $P_H(\tX)$ can be done iteratively in multiple ways by finding a sequence of contraction paths for $H$ where we start with each edge of $H$ endowed with $\tX$ and our end goal is to have a single black edge aligned with the red edge.
 Figure~\ref{fig:computation_of_P_H} provides an example of computing a 4th degree polynomial,
 \begin{equation*}
     P_H(\tX)_{j_1,j_2} = \sum_{j_3,j_4=1}^n\tX_{j_1j_4}\tX_{j_1j_3}\tX_{j_3j_2}\tX_{j_1j_2}.
 \end{equation*}
 The computation of the polynomial is decomposed to a sequence of operations, portrayed in the figure. Each step is labeled by the tensor contraction operation and the corresponding explicit computation. Nodes colored in gray correspond to contracted nodes whose indices are summed in the $\texttt{einsum}$. The output of each contraction step is represented by a new black edge (labeled as $\tY,\tZ$ and $\tR$ in our example).
\subsection{Generalizations and discussion}\label{ss:gen_and_discussion}

\textbf{Invariant graph polynomials.}
This approach also gives a basis for the invariant polynomials $P: \Real^{n^2} \to \Real$. In this case, we let $H$ be a directed multigraph without a red edge, and define $P_H(\tX) = \sum_{j_1, \ldots, j_m \in [n]} \prod_{r,s \in E(H)} \tX_{j_r,j_s}$. Computing $P_H$ then corresponds to contracting $H$ to the trivial graph (with no nodes or edges). Our equivariant basis is a generalization of previous work, which used invariant polynomials analogous to $P_H$ or the alternative basis $Q_H$ to study properties of graphs~\citep{thiery2000algebraic, lovasz2012large, komiske2018energy}.

\textbf{Subgraph counting.}
The previous work on invariant polynomials mentioned above as well as our proof of Theorem~\ref{thm:poly} suggest $Q_H$ (see \eqref{e:Q_H}) as another basis of equivariant graph polynomials. In Appendix~\ref{appendix:homomorphism}, we show that when applied to binary input $\tX \in \{0, 1\}^{n \times n}$, $Q_H$ performs subgraph counting; essentially, $Q_H(\tX)_{i_a, i_b}$ is proportional to the number of subgraphs of $\tX$ isomorphic to $H$ such that $i_a$ is mapped to $a$ and $i_b$ is mapped to $b$. This $Q_H$ basis is interpretable, but does not lend itself to efficient vectorized computation or the tensor contraction perspective that the $P_H$ basis has.

\textbf{Equivariant polynomials for attributed graphs.}
Our basis for equivariant graph polynomials can be extended to cover the more general case of \emph{attributed graphs} (\ie, graphs with $\Real^f$ features attached to nodes and/or edges), $P:\Real^{n^2\times f}\too\Real^{n^2}$. A similar basis to $P_H$ can be used in this case, as described in \Cref{a:attributed_graphs}. \Cref{fig:attr_poly} visualizes this extension. 
\begin{figure}[H]
    \centering      \resizebox{\columnwidth}{!}{
    \begin{tabular}{cc}
  \makecell[c]{\begin{tikzpicture}[main/.style = {draw, circle},baseline=5pt,semithick] 
\node[main] (1) [minimum size=18pt] {1};
\node[main] (2) [above of=1, minimum size=18pt] {2};
\node[main] (3) [right of=1, minimum size=18pt] {3};
\draw[<-,color = green] (1) edge (2);
\draw[<-, bend left=15,color = blue] (3) edge (1);
\draw[<-, bend left=15, bend right=15,color = orange] (3) edge (1);
\path (2) edge [densely dotted, loop above, color=red] (2);
\end{tikzpicture}\\ \  }  &\makecell[bl]{ $P_H(\tX)_{j_2j_2}=$\\[5pt]
$\sum_{j_1,j_3=1}^n \tX_{j_2j_1\textcolor{green}{1}} \tX_{j_1j_3\textcolor{orange}{0}}\tX_{j_1j_3\textcolor{blue}{2}}=$\\[5pt]
      $\texttt{einsum}(21,13,13 \rightarrow 22 ,\tX[:,:,\textcolor{green}{1}],\tX[:,:,\textcolor{orange}{0}],\tX[:,:,\textcolor{blue}{2}])_{j_2j_2}$ \vspace{-22pt}  }
     \vspace{-7pt}\\\hline  \vspace{-7pt}     \\
\makecell[c]{
\begin{tikzpicture}[main/.style = {draw, circle},baseline=5pt,semithick] 
\node[main] (1) [minimum size=18pt] {1};
\node[main] (2) [above of=1, minimum size=18pt] {2};
\node[main] (3) [right of=1, minimum size=18pt] {3};
\draw[<-, bend right = 15, densely dotted, color = red] (1) edge (2);
\draw[<-, bend left = 15, color = green] (1) edge (2);
\draw[->, color=orange, bend right=15] (2) edge (3);
\draw[->, bend left=15, color=orange] (2) edge (3);
\end{tikzpicture}   }  &\makecell[bl]{  $P_H(\tX)_{j_2j_1}=$\\[5pt]
$\sum_{j_3=1}^n \tX_{j_2j_3\textcolor{orange}{0}}^2 \tX_{j_2j_1\textcolor{green}{1}}=$\\[5pt]
      $\texttt{einsum}(23,23,21 \rightarrow 21 ,\tX[:,:,\textcolor{orange}{0}],\tX[:,:,\textcolor{orange}{0}],\tX[:,:,\textcolor{green}{1}])_{j_2j_1}$\vspace{-22pt} }
    \end{tabular}
    }
    \caption{Basis elements of equivariant polynomials from $\mathbb{R}^{n^2 \times 3} \to \mathbb{R}^{n^2}$. The output edge is  indicated by a dotted red edge and the feature dimension is indexed by three colors for index zero (orange), index one (green), and index two (blue).}\label{fig:attr_poly} 
\end{figure}



\section{Expressive Power of Graph Models}
In this section we evaluate the expressive power of equivariant graph models from the new, yet natural hierarchy arising from equivariant graph polynomials. By graph model, $\gF=\set{F}$, we mean any collection of equivariant functions $F:\Real^{n^2}\too \Real^{n^k}$, where $k=1$ corresponds to a family of node-valued functions, $F(\tX)\in\Real^n$, and $k=2$ to node and edge-valued functions, $F(\tX)\in\Real^{n^2}$. For expositional simplicity we focus on graph data $\tX$ representing simple graphs, but note that the general graph data case can be analysed using similar methods. We will use two notions of polynomial expressiveness: exact and approximate. The exact case is used for analyzing Prototypical graph models, whereas the approximate case is used for analyzing practical graph models.
\begin{definition}\label{def:exact}
A graph model $\gF$ is $d$ node/edge polynomial \emph{exact} if it can compute all the degree $d$ polynomial basis elements $P_H(\tX)$ for every simple graph $\tX$.
\end{definition}
\begin{definition}\label{def:expressive}
A graph model $\gF$ is $d$ node/edge polynomial \emph{expressive} if for arbitrary $\eps>0$ and degree $d$ polynomial basis element $P_H(\tX)$ there exists an $F\in \gF$ such that $\max_{\tX\text{ simple}} \abs{F(\tX)-P(\tX)} < \eps$.
\end{definition}

As a primary application of the equivariant graph basis $P_H$, we develop tools here for analyzing the polynomial expressive power of graph models $\gF$. We define \emph{Prototypical} graph models which provide a structure to analyze or improve existing popular GNNs such as MPNN \citep{gilmer2017neural} and PPGN~\citep{maron2019provably}.

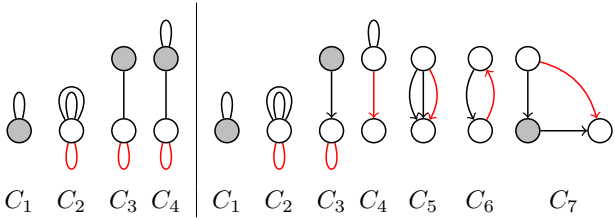
\begin{figure}[t]
    \centering
    \resizebox{\columnwidth}{!}{
    \begin{tabular}{@{\hskip0pt}c@{\hskip2pt}c@{\hskip2pt}c@{\hskip5pt}c|c@{\hskip2pt}c@{\hskip2pt}c@{\hskip5pt}c@{\hskip5pt}c@{\hskip5pt}c@{\hskip5pt}c@{\hskip0pt}}
   \begin{tikzpicture}[main/.style = {draw, circle},baseline=5pt,semithick,every loop/.style={}]
\node[main] (1) [fill=lightgray] {};
\path 
(1) edge [-, loop above] node {} (1);
\end{tikzpicture}
&  
 \begin{tikzpicture}[main/.style = {draw, circle},baseline=5pt,semithick,every loop/.style={}]
\node[main] (1) [] {};
\path 
(1) edge [-, loop below, color=red] node {} (1)
(1) edge [-, loop above] node {} (1)
(1) edge [out=120, in=60, loop, min distance = 20pt] (1);
\end{tikzpicture}
& 
\begin{tikzpicture}[main/.style = {draw, circle},baseline=5pt,semithick,every loop/.style={}]
\node[main] (2) [] {};
\node[main] (1) [above of=2, fill=lightgray] {};
\path 
(1) edge [-] node {} (2)
(2) edge [-, loop below, color=red] node {} (2);
\end{tikzpicture} & 
     \begin{tikzpicture}[main/.style = {draw, circle},baseline=5pt,semithick,every loop/.style={}]
\node[main] (2) [] {};
\node[main] (1) [above of=2, fill=lightgray] {};
\path 
(1) edge [loop above,-] node {} (1)
(1) edge [-] node {} (2)
(2) edge [-, loop below, color=red] node {} (2);
\end{tikzpicture} 
&
\begin{tikzpicture}[main/.style = {draw, circle},baseline=5pt,semithick,every loop/.style={}]
\node[main] (1) [fill=lightgray] {};
\path 
(1) edge [-, loop above] node {} (1);
\end{tikzpicture}
&  
 \begin{tikzpicture}[main/.style = {draw, circle},baseline=5pt,semithick,every loop/.style={}]
\node[main] (1) [] {};
\path 
(1) edge [-, loop below, color=red] node {} (1)
(1) edge [-, loop above] node {} (1)
(1) edge [out=120, in=60, loop, min distance = 20pt] (1);
\end{tikzpicture}
& 
\begin{tikzpicture}[main/.style = {draw, circle},baseline=5pt,semithick,every loop/.style={}]
\node[main] (2) [] {};
\node[main] (1) [above of=2, fill=lightgray] {};
\path 
(1) edge [->] node {} (2)
(2) edge [-, loop below, color=red] node {} (2);
\end{tikzpicture} 
&
\begin{tikzpicture}[main/.style = {draw, circle},baseline=5pt,semithick,every loop/.style={}]
\node[main] (2) [] {};
\node[main] (1) [above of=2] {};
\path 
(1) edge [-, loop above] node {} (1)
(1) edge [->, color=red] node {} (2);
\end{tikzpicture}
& 
\begin{tikzpicture}[main/.style = {draw, circle},baseline=5pt,semithick,every loop/.style={}]
\node[main] (2) [] {};
\node[main] (1) [above of=2] {};
\path 
(1) edge [->] node {} (2)
(1) edge [->, bend right] node {} (2)
(1) edge [->, bend left, color=red] node {} (2);
\end{tikzpicture}
& 
\begin{tikzpicture}[main/.style = {draw, circle},baseline=5pt,semithick,every loop/.style={}]
\node[main] (2) [] {};
\node[main] (1) [above of=2] {};
\path 
(1) edge [->, bend right] node {} (2)
(1) edge [<-, bend left, color=red] node {} (2);
\end{tikzpicture}
& 
\begin{tikzpicture}[main/.style = {draw, circle},baseline=5pt,semithick,every loop/.style={}]
\node[main] (2) [fill=lightgray] {};
\node[main] (1) [above of=2] {};
\node[main] (3) [right of=2] {};
\path 
(1) edge [->] node {} (2)
(2) edge [->] node {} (3)
(1) edge [->, bend left, color=red] node {} (3);
\end{tikzpicture}\\
$C_1$ & $C_2$ & $C_3$ & $C_4$ & $C_1$ & $C_2$ & $C_3$ & $C_4$ & $C_5$ & $C_6$ & $C_7$
\end{tabular} }
    \caption{Prototypical node-based (left) and edge-based (right)  graph models' contraction banks. Gray nodes indicate nodes that are contracted. Explicit formula of each element can be found in Appendix \ref{s:approximation_proof}. \vspace{-5pt}}
    \label{fig:ideal_graph_models}
\end{figure}
\vspace{-5pt}

\subsection{Prototypical graph models }
We consider graph computation models, $\gF_\gB$, that are finite sequences of tensor contractions taken from a bank of primitive contractions $\gB$. 
\begin{equation}
    \gF_\gB=\set{C_{i_1} C_{i_2} \cdots C_{i_\ell} \ \vert C_{i_j} \in \gB },
\end{equation}
where the bank 
    $\gB=\set{C_1,\ldots,C_k}$ 
consists of multi-graphs $C_i=(V_i,E_i,(a_i,b_i))$, each representing a different primitive tensor contraction. 
%
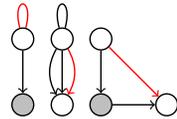
\begin{wrapfigure}[9]{r}{.30\columnwidth}
  \begin{center}\vspace{-12pt}\hspace{-10pt}
  \resizebox{.3\columnwidth}{!}{
  \begin{tabular}{@{\hskip3pt}c@{\hskip3pt}c@{\hskip3pt}c@{\hskip0pt}}
   \begin{tikzpicture}[main/.style = {draw, circle},baseline=5pt,semithick,every loop/.style={}]
\node[main] (2) [fill=lightgray] {};
\node[main] (1) [above of=2] {};
\path 
(1) edge [-, loop above, color=red] node {} (1)
(1) edge [->] node {} (2);
\end{tikzpicture}
& 
\begin{tikzpicture}[main/.style = {draw, circle},baseline=5pt,semithick,every loop/.style={}]
\node[main] (2) [] {};
\node[main] (1) [above of=2] {};
\path 
(1) edge [->] node {} (2)(1) edge [-, loop above] node {} (1)
(1) edge [->, bend right] node {} (2)
(1) edge [->, bend left, color=red] node {} (2);
\end{tikzpicture}
&
   \begin{tikzpicture}[main/.style = {draw, circle},baseline=5pt,semithick,every loop/.style={}]
\node[main] (1) [] {};
\node[main] (2) [left of=1,fill=lightgray] {};
\node[main] (3) [above of=2] {};
\path 
(3) edge [->] node {}  (2)
(2) edge [->] node {}  (1)
(3) edge [->, color=red] node {} (1); 
\end{tikzpicture}
    \end{tabular} }
    \end{center}\vspace{-10pt}
  \caption{A bank $\gB$ of a model $\gF_\gB$ that can compute the example in Figure \ref{fig:computation_of_P_H}.}\label{fig:F_simple}
\end{wrapfigure}
A model $\gF_\gB$ can compute a polynomial $P_H(\tX)=\texttt{einsum}(H,\tX,\ldots,\tX)$ 
if it can contract $H$ to the red edge by applying a finite sequence of contractions from its bank. 
If there exists such a sequence then $P_H$ is deemed computable by $\gF_\gB$, otherwise it is not computable by $\gF_\gB$. For example, the model with the bank presented in Figure \ref{fig:F_simple} can compute $P_H$ in Figure \ref{fig:computation_of_P_H}; removing any element from this model, will make $P_H(\tX)$ non-computable. We recap:
\vspace{5pt}
\begin{definition}
    The polynomial $P_H$ is \emph{computable by} $\gF$ iff there exists a sequence of tensor contractions from $\gF$ that computes $P_H(\tX)=\texttt{einsum}(H,\tX,\ldots,\tX)$.
\end{definition}
We henceforth focus on two Prototypical models: the node-based model $\gF_{n}$ and edge-based model $\gF_{e}$. Their respective contraction banks are depicted in Figure \ref{fig:ideal_graph_models}, each motivated by the desire to achieve polynomial exactness (see Definition \ref{def:exact}) and contract multi-graphs $H$ where a member of the bank can always be used to contract nodes with up to $N$ neighbors. Taking $N=1$ results in the node-based bank in Figure \ref{fig:ideal_graph_models} (left), and $N=2$ in the edge-based bank in Figure \ref{fig:ideal_graph_models} (right). These choices are not unique --- other contraction banks can satisfy these requirements.
\begin{lemma}\label{lem:always_can_contract_verts}
    $\gF_n$ (for simple graphs) and $\gF_e$ (for general graphs) can always contract a node in $H$ iff its number of neighbors is at-most $1$ and $2$, respectively. 
\end{lemma}
A few comments are in order. Node-based contractions can only add self-edges during the contraction process (\ie, new node-valued data) thus requiring only $O(n)$ additional memory to perform computation. Further note that since we assume simple graph data, $H$ is also a simple graph, and no directed edges (\ie, non-symmetric intermediate tensors) are created during contraction. Contraction banks with undirected graphs suffice in this setting. We later show that the node-based model acts analogously to message-passing. 
The edge-based model targets exactness over both node and edge valued polynomial. It generates new edges that can be directed even if $\tX$ is simple, and thus includes directed contractions in its bank. The edge-based model
will later be connected to the graph models PPGN~\citep{maron2019provably} and 
Ring-GNN~\citep{chen2019equivalence}. We later show that the node-based model is 1-WL expressive and the edge-based model is 3-WL expressive.

\begin{algorithm}[tb]
   \caption{Decide if $P_H$ is computable by $\gF_{n}$ or $\gF_{e}$.}
   \label{alg:Q_H}
\begin{algorithmic}
   \STATE {\bfseries Input:} contraction bank $\gF\in\set{\gF_{n},\gF_e}$, multi-graph $H$ 
   \item set $d=1$ for $\gF=\gF_n$, or $d=2$ for $\gF=\gF_e$.
   \item set $doneContracting=false$
   \WHILE{not $doneContracting$ }
   \IF{exists a node in $V\setminus \set{a,b}$ with $\leq d$ neighbors}
   \item contract the node using $\gF$ 
   \ELSE
   \item $doneContracting=true$
   \ENDIF   
   \ENDWHILE
   \IF{$V\setminus \set{a,b}$ is empty}
   \item return \texttt{computable}
   \ELSE
   \item return \texttt{non-computable}
   \ENDIF   
\end{algorithmic}
\end{algorithm}

\textbf{Deciding computability of $P_H(\tX)$ with $\gF_\gB$.}
A key component in analyzing the expressive power of a Prototypical model is determining which polynomials $P_H(\tX)$ can be computed with $\gF$, given $H$ and $\tX$ encoding simple graph data. A naive algorithm traversing all possible enumerations of nodes in $H$ and their contractions would lead to a combinatorial explosion that is too costly --- especially since this procedure needs to be repeated for a large number of polynomials. Here, we show that at least for contraction banks $\gF_n$ and $\gF_e$, Algorithm \ref{alg:Q_H} is a linear time (in $|V|,|E|$), \emph{greedy} algorithm for deciding computability of a given polynomial. Algorithm \ref{alg:Q_H} finds a sequence of contractions using the greedy step until no more nodes are left to contract. That is, it terminates when all nodes, aside from $a,b$, have more than 1 or 2 neighbors for $\gF_n$ or $\gF_e$, respectively. If it terminates with just $\set{a,b}$ as vertices it deems $P_H$ computable and otherwise it deems $P_H$ non-computable. To show correctness of this algorithm we prove:
\begin{theorem}\label{thm:contraction}
    Let $H$ be some multi-graph and $\gF_\gB\in\set{\gF_n,\gF_e}$. Further, let $H'$ be the multi-graph resulting after contracting a single node in $H$ using one or more operations from $\gB$ to $H$. Then, $H$ is $\gF_\gB$-computable iff $H'$ is $\gF_\gB$-computable. 
\end{theorem}

To verify the correctness of this procedure, note that the algorithm has to terminate after at most $|V|-|\set{a,b}|$ node contractions. Now consider two cases: if the algorithm terminates successfully, it must have found a sequence of tensor contractions to compute $P_H(\tX)$. If it terminates unsuccessfully, the theorem implies its last network $H'$ is computable iff the input network $H$ is computable. Now since there is \emph{no} further node contraction possible to do in $H'$ using operations from $\gB$ it is not computable by definition, making $H$ not computable.

\begin{wrapfigure}[10]{r}{.4\columnwidth}
  \begin{center}\vspace{-20pt}\hspace{-10pt}
  \resizebox{.40\columnwidth}{!}{
  \begin{tabular}{@{\hskip0pt}c@{\hskip3pt}c@{\hskip3pt}c@{\hskip0pt}}
    
    \begin{tikzpicture}[main/.style = {draw, circle},baseline=5pt,semithick,every loop/.style={}]
\node[main] (1) [] {};
\node[main] (2) [above of=1] {};
\node[main] (3) [right of=1] {};
\draw[-] (1) edge (2);
\draw[-] (2) edge (3);
\draw[-] (3) edge (1);
\path (2) edge [-,loop above, color=red] (2);
\end{tikzpicture} 
        & 
\begin{tikzpicture}[main/.style = {draw, circle},baseline=5pt,semithick,every loop/.style={}]
\node[main] (1) [] {};
\node[main] (2) [above of=1] {};
\node[main] (3) [right of=1] {};
\node[main] (4) [above of=3] {};
\draw[-, loop above, color=red] (2) edge (2);
\draw[-] (1) edge (2);
\draw[-] (1) edge (3);
\draw[-] (1) edge (4);
\draw[-] (2) edge (3);
\draw[-] (2) edge (4);
\draw[-] (4) edge (3);
\end{tikzpicture} 
& 
\begin{tikzpicture}[main/.style = {draw, circle},baseline=5pt,semithick,every loop/.style={}]
\node[main] (1) [] {};
\node[main] (2) [above of=1] {};
\node[main] (3) [right of=1] {};
\node[main] (4) [above of=3] {};
\draw[<-, color=red] (1) edge (2);
\draw[-] (1) edge (3);
\draw[-] (1) edge (4);
\draw[-] (2) edge (3);
\draw[-] (2) edge (4);
\draw[-] (4) edge (3);
\end{tikzpicture}
\end{tabular} }
    \end{center}\vspace{-10pt}
  \caption{The smallest non-computable $H$ for: $\gF_n$ (left: node-valued), and $\gF_e$ (middle: node-valued; right: edge-valued).}\label{fig:failure_H_for_Fn_Fe}
\end{wrapfigure}
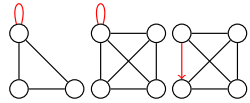
\textbf{Polynomial exactness.} To compute the $d$ polynomial exactness (see Definition \ref{def:exact}) for the node-based and edge-based Prototypical graph models we enumerate all non-isomorphic simple graphs $H$ with up to $d$ edges and one red edge and run Algorithm \ref{alg:Q_H} on each $H$. This reveals that the node-based model $\gF_n$ is 2-node-polynomial-exact, while the edge-based model $\gF_e$ is 5-node-polynomial-exact and 4-edge-polynomial-exact. See Figure \ref{fig:failure_H_for_Fn_Fe} for the lowest degree polynomials, represented by $H$ with the smallest number of edges, that are non-computable for $\gF_n$ and $\gF_e$. 

\textbf{$k$-WL expressive power.} For simple graphs, there is a natural connection between our Prototypical graph models and the $k$-WL graph isomorphism tests. This stems from a result of \citet{dvovrak2010recognizing, dell2018lov}, which states that two graphs $\tX^{(1)}$ and $\tX^{(2)}$ are $k$-FWL equivalent if and only if $\hom(H, \tX^{(1)}) = \hom(H, \tX^{(2)})$ for all $H$ of tree-width at most $k$. Recall that $\hom(H, \tX)$ is the number of homomorphisms from $H$ to $\tX$ (where $H$ has no red edges), which we show is equivalent to the output of $P_H(\tX)$ for the invariant polynomial $P_H$ in Appendix~\ref{appendix:homomorphism}. By showing that the Prototypical node-based graph model can contract any $H$ of tree-width 1 (and no others), and that the Prototypical edge-based graph model can contract any $H$ of tree-width at most 2 (and no others), we thus have the following result.
\begin{proposition}\label{prop:tree-width}
    The Prototypical node-based model can distinguish a pair of simple graphs if and only if  1-WL can. The Prototypical edge-based model can distinguish a pair of simple graphs if and only if 3-WL / 2-FWL can.
\end{proposition}

This proposition indicates that $\mathcal{F}_n$ and $\mathcal{F}_e$ can contract an invariant polynomial, represented by a graph $H$, if and only if the tree-width of $H$ is $1$ and $2$, respectively. Therefore computability of \emph{invariant} $P_H$ can be decided by checking the tree width of $H$. We leave generalizing this approach to equivariant $P_H$ to future work.

\subsection{GNNs and their expressive power}\label{ss:gnn_expressive_power}
In this section we turn our attention to commonly used graph neural networks (GNNs), and provide lower bounds on their polynomial expressive power as in Definition \ref{def:expressive}. The Message Passing Neural Network (MPNN) we consider consists of layers of the form
\begin{equation}
    \tY^{(k+1)} = \texttt{m}\brac{\tX \tY^{(k)}, \one\one^T \tY^{(k)}, \tY^{(k)}},
    \label{eq:MPNN}
\end{equation}
where the intermediate tensor variables are $\tY\in\Real^{n\times d}$, $\one\in\Real^n$ is the vector of all ones, $\tY^{(0)}=\one$, brackets indicate concatenation in the feature dimension, and $\texttt{m}$ means a multilayer perceptron (MLP) applied to the feature dimension.  

As an application of the Prototypical edge-based model, we propose and implement a new model architecture (PPGN++) that is at least as expressive as the full versions of PPGN/Ring-GNN \cite{maron2019provably, chen2019equivalence} (which incorporate all $15$ linear basis elements), but is more efficient --- PPGN++ uses a smaller number of ``primitive'' operations than the full PPGN/Ring-GNN, and does not need parameters for each linear basis element:
\begin{equation}
    \tZ^{(k+1)} = \bar{\texttt{m}}_3\brac{ \bar{\texttt{m}}_1\brac{\tZ^{(k)},\tZ^{(k)T}} \circledast
 \bar{\texttt{m}}_2(\tZ^{(k)}), \tZ^{(k)}},
    \label{eq:PPGN++}
\end{equation} 
where $\tZ\in\Real^{n^2\times d}$ are intermediate tensor variables, $\tZ^{(0)}= \tX$, $\circledast$ performs matrix multiplication of matching features, and $\bar{\texttt{m}}_i$, for $i\in[3]$, is a \emph{pair} of MLPs:   
one applied to all diagonal and off-diagonal features of $\tZ$ separately. 

We lower bound the polynomial expressiveness of MPNN and PPGN++ in the next theorem:
\begin{theorem}\label{thm:gnn_poly_lower_bound}
    PPGN++ is at-least $4$ edge polynomial expressive and $5$ node polynomial expressive. MPNN is at-least $2$ node polynomial expressive.
\end{theorem}
\textit{Proof idea.} We prove the theorem in two steps. First, showing that an MPNN or PPGN++ layer can approximate any primitive contraction $C\in\gB$ from the bank of the Prototypical node based $\gF_n$ or edge based $\gF_e$ models, respectively. Second, we use a lemma from \citet{lim2022sign} stating that layer-wise universality leads to overall universality. The complete proof is in Appendix \ref{s:approximation_proof}.

\textbf{Comparison of PPGN++ and PPGN.}
Proposition \ref{prop:tree-width} and the proof of Theorem \ref{thm:gnn_poly_lower_bound}  indicate that PPGN++ is 3-WL/2-FWL expressive for simple graphs, similarly to PPGN \citep{maron2019provably}. 
However, the following proposition shows that there is a significant expressiveness gap between PPGN and PPGN++ in approximating equivariant polynomials.
\begin{proposition}\label{thm:ppgn_poly_lower_bound}
    PPGN is at most $0$ edge polynomial expressive. 
\end{proposition}
\textit{Proof idea.} We claim that PPGN is at most $0$ edge polynomial expressive by proving that there exist a linear polynomial (the transpose operator) that cannot be approximated by PPGN. The proof shows that for an input tensor of the form
\begin{equation*}  
\tZ=\begin{bmatrix}
    a & a \\
    b & b \\
\end{bmatrix}, \quad a,b\in\Real,
\end{equation*}
a PPGN model cannot approximate the transpose operator $P_H(\tZ)=\tZ^T$ since it preserves the row structure. The complete proof is in Appendix \ref{s:ppgn_approximation_proof}.


\subsection{Increasing the expressive power of GNNs}\label{ss:poly_features}
Theorem \ref{thm:gnn_poly_lower_bound} proves a lower bound on the polynomial expressiveness of two popular GNN models --- a natural question is how to increase the expressiveness beyond the lower bound. Polynomial expressiveness provides a simple path forward to add network operations or input features that complement these architectures with polynomials that are otherwise uncomputable. In our study, we add input features to enhance expressiveness.  

Suppose we have a $d'$ polynomial expressive GNN model (with a corresponding $d'$ exact Prototypical graph model $\gF_\gB$) that we want to extend it to be $d>d'$ polynomial expressive. For every $\ell\in\sN$, $d'+1\leq \ell \leq d$, we can compute all $\gF_\gB$ \emph{non-computable} $\ell$-degree basis elements of $P_H$ using Algorithm \ref{alg:Q_H}, considering all non-isomorphic, simple and connected $H$. Indeed any $H$ with two disconnected components corresponds to a multiplication of two lower degree polynomials  approximable by the GNN itself (or using lower degree polynomial features). Any non-computable polynomials discovered in this process are added as node/edge input features to the architecture, effectively increasing the polynomial expressiveness of the model to $d$. 


\begin{table}[h!]
\centering
\vspace{-5pt}
\caption{Numbers of non-computable polynomials (left) out of all relevant polynomials (right) for the Prototypical models. }\label{tab:computable}
\vspace{5pt}
\resizebox{0.8\columnwidth}{!}{
\begin{tabular}{c|c|c|c|c|c}\toprule   & $d=3$ & $d=4$ & $d=5$ & $d=6$ & $d=7$ \\\hline
$\gF_n$ & 2/8 & 6/18 & 23/49 & 85/144 & 308/446 \\  \hline
$\gF_e$ & 0/18 & 0/53 & 1/174 & 11/604& 72/2193\\    \bottomrule
\end{tabular}}
\end{table} 
In Table \ref{tab:computable} we list, for each Prototypical model and degree $d$, the number of polynomials that are found non-computable by the Prototypical models (left), out of the total number of relevant polynomials $P_H$ (right). For the node based model we count only node-valued polynomials, while for the edge based model we count both node and edge-valued polynomials. 
Note that the number of non-computable polynomials is substantially smaller than the total number, especially in $\gF_e$. 
Since polynomials are calculated at the data preprocessing step, there is an upfront computational cost for this procedure that must be accounted for. Finding the optimal contraction path that minimizes runtime complexity for a general matrix polynomial is an NP-hard problem \citep{Biamonte_2015Tensorcontraction} with a naive upper bound in runtime complexity of $O(n^d)$. An empirical evaluation of the preprocessing time is in Appendix \ref{s:imp_details}; in our experiments, preprocessing time is small compared to training time.

\section{Related Work}

\textbf{Relation to Homomorphisms and Subgraph Counts.} Past work has studied invariant polynomials on graphs \citep{thiery2000algebraic, lovasz2012large, komiske2018energy}. Viewed as functions on binary inputs, the basis consists of functions that count homomorphisms or injective homomorphisms of $H$ into an input graph $\tX$. Homomorphisms are related to the $P_H$ basis, and injective homomorphisms are related to $Q_H$ (see Appendix~\ref{appendix:homomorphism}). Also, equivariant homomorphism counts that relate to our $P_H$ or $Q_H$ has been studied~\citep{manvcinska2020quantum, grohe2021homomorphism, maehara2019simple, bouritsas2022improving, Barcelo2021LocalGraphParameters, welke2022expectation}. However, these works do not exhibit a basis of equivariant polynomials. Also, our tensor contraction interpretation and analysis does not appear in past work.

\textbf{Expressivity Measures for Graph Models.} The $k$-WL hierarchy has been widely used for studying graph machine learning~\citep{morris2021weisfeiler}, starting with the works of \citet{morris2019weisfeiler} and \citet{xu2019how}, which show an equivalence between message passing neural networks and 1-WL. Tensor methods resembling $k$-WL such as $k$-IGN~\citep{maron2018invariant} and PPGN-like methods~\citep{maron2019provably, azizian2021expressive} achieve $k$-WL power~\citep{azizian2021expressive, geerts2022expressiveness}, but scale in memory as $n^k$ or $n^{k-1}$ for $n$-node graphs. \citet{morris2019sparsewl, morris2022speqnet} define new $k$-WL variants with locality and sparsity biases, which gives a finer hierarchy and offers a trade-off between efficiency and expressiveness.

Various works measure the expressivity of graph neural networks by the types of subgraphs that they can count~\citep{chen2020can, tahmasebi2020counting, arvind2020weisfeiler}. On simple graphs, subgraph counting of $H$ is equivalent to evaluating an invariant polynomial $Q_H$. Additional works have studied the ability of graph models to compute numerous other graph properties. For instance, graph machine learning models have been studied in the context of approximating combinatorial algorithms~\citep{sato2019approximation}, solving biconnectivity problems~\citep{zhang2023rethinking}, computing spectral invariants~\citep{lim2022sign}, distinguishing rooted graphs at the node level~\citep{chen2021on}, and computing various other graph properties~\citep{garg2020generalization}. As opposed to our framework, these expressivity measures generally do not induce a hierarchy of increasing expressivity, and they often do not directly suggest improvements for graph models

A matrix query language (MATLANG)~\citep{brijder2019expressive, geerts2021expressive} and a more general tensor language (TL)~\citep{geerts2022expressiveness} have been used to study expressive power of GNNs~\citep{balcilar2021breaking, geerts2022expressiveness}. These languages define operations and ways to compose them for processing graphs in a permutation equivariant or invariant way. Our edge-based Prototypical model result gives a new perspective on a result of \citet{geerts2021expressive}, which shows that MATLANG can distinguish any two graphs that 2-FWL / 3-WL can. Indeed, our edge-based graph model includes the five linear algebra operations that form the 3-WL expressive MATLANG. While the operations of MATLANG were included in a somewhat ad-hoc manner (``motivated by operations supported in linear algebra package'' \citep{geerts2021expressive}), our framework shows that these are the at-most quadratic equivariant polynomials that are required to contract all tree-width 2 graphs.

\textbf{Other Expressive GNNs.}
Various approaches have been used to develop expressive graph neural networks. One approach adds node or edge features, oftentimes positional or structural encodings, to base graph models~\citep{sato2021random, abboud2021surprising, bouritsas2022improving, lim2022sign, zhang2023rethinking, li2020distance, Loukas2020What}. Subgraph GNNs treat an input graph as a collection of subgraphs~\citep{bevilacqua2022equivariant, frasca2022understanding, qian2022ordered, cotta2021reconstruction, Zhao2021akgnn, you2021identity, zhang2021nested}. Some models utilize modified message passing and higher-order convolutions \citep{bodnar2021cell, Bodnar2021topological, thiede2021autobahn, de2020natural}. One can also take a base model and perform group averaging or frame averaging to make it have the desired equivariances while preserving expressive power~\citep{murphy2019relational, puny2022frame}.

\section{Experiments}\label{s:experiments}
In this section we demonstrate the impact of increasing the polynomial expressive power of GNNs. We test two families of models. PPGN++ ($d$) uses the architecture in \eqref{eq:PPGN++}, derived using our edge based Prototypical model, and achieves $d$ polynomial expressive power by pre-computing polynomial features found in Subsection \ref{ss:poly_features}; missing ($d$) notation means using just PPGN++ without pre-computed features. GatedGCN ($d$) uses the base MPNN architecture of \cite{Bresson2017Gated} with the $d$-expressive polynomials pre-computed.
We experiment with $4$ datasets: a graph isomorphism dataset SR \citep{Bodnar2021topological}, which measures the ability of GNNs to distinguish strongly regular graphs; and $3$ real-world molecular property prediction datasets including ZINC, ZINC-full \citep{Dwivedi2020benchmarking} and Alchemy \citep{Chen2019Alchemy}. 

\subsection{Graph Isomorphism Expressiveness}
Distinguishing non-isomorphic graphs from families of strongly regular graphs is a challenging task \citep{bodnar2021cell, Bodnar2021topological}. The SR dataset \citep{bouritsas2022improving} is composed of 9 strongly regular families. This dataset is challenging since any pair of graphs in the SR dataset cannot be distinguished by the $3$-WL algorithm. This experiment is done without any training (same procedure as in \citep{Bodnar2021topological}) and the evaluation is done by randomly initialized models. For every family in the dataset, we iterate over all pairs of graphs and report the fraction that the model determines are isomorphic. Two graphs are considered isomorphic if the $L_2$ distance between their embeddings is smaller than a certain threshold ($\epsilon=0.01$). This procedure was repeated for $5$ different random seeds and the averaged fraction rate was reported in Figure \ref{fig:sr}. This figure portrays the expressiveness boost gained by using high degree polynomial features. While the base models, GatedGCN and PPGN (and PPGN++), cannot distinguish any pair of graphs (as theoretically expected), adding higher degree polynomial features significantly improves the ability of the model to distinguish non-isomorphic graphs in this dataset. Optimal results of $0\%$ failure rate are obtained for PPGN++ ($d$) with $d\geq 6$ (\ie, adding at-least degree $6$ polynomial features). For the GatedGCN model, although we do not reach the $0\%$ failure rate, adding the right polynomial features makes GatedGCN outperform $3$-WL based models in distinguishing non-isomorphic graphs in this dataset. 
 

\begin{table}[h!]
\renewcommand{\tabcolsep}{1.6pt}
\small
\centering
\caption{Results on Alchemy \citep{Chen2019Alchemy} and ZINC-full \citep{Dwivedi2020benchmarking} datasets. Lower is better, best models are marked in \textbf{bold}.}
\vspace{5pt}
\begin{tabular}{l|c|c}
\toprule
\multirow{2}{*}{Model} & ZINC-Full & Alchemy \\ \cline{2-3} 
 & Test MAE & Test MAE \\ \hline
GIN \citep{xu2019how} & $.088 \pm .002$ & $.180\pm .006$ \\
$\delta$-$2$-GNN \citep{morris2019sparsewl} & $.042 \pm .003$ & $.118\pm .001$ \\
SpeqNet \citep{morris2022speqnet} & - & $.115\pm .001$ \\
PF-GNN \citep{dupty2022pfgnn} & - & $.111 \pm .010$ \\
HIMP \citep{Fey2020himp} & $.036 \pm .002$ & - \\
SignNet \citep{lim2022sign} & $.024 \pm .003$ & $.113 \pm .002$ \\
CIN \citep{bodnar2021cell} & $.022 \pm .002$ & - \\ \hline
PPGN \citep{maron2019provably} &  $.022 \pm .003$ & $.113 \pm .001$ \\ 
PPGN++ & $.022 \pm .001$ & $.111 \pm .002$ \\
\textbf{PPGN++ (5)} & $\bm{.020 \pm .001}$ & $.110 \pm .001$ \\
\textbf{PPGN++ (6)} & \multicolumn{1}{l|}{$\bm{.020 \pm .001}$} & $\bm{.109 \pm .001}$ \\ \bottomrule
\end{tabular}
\label{tab:alchemy+zinc_full}
\end{table}
\begin{figure}[h!]
    \includegraphics[width = \columnwidth]{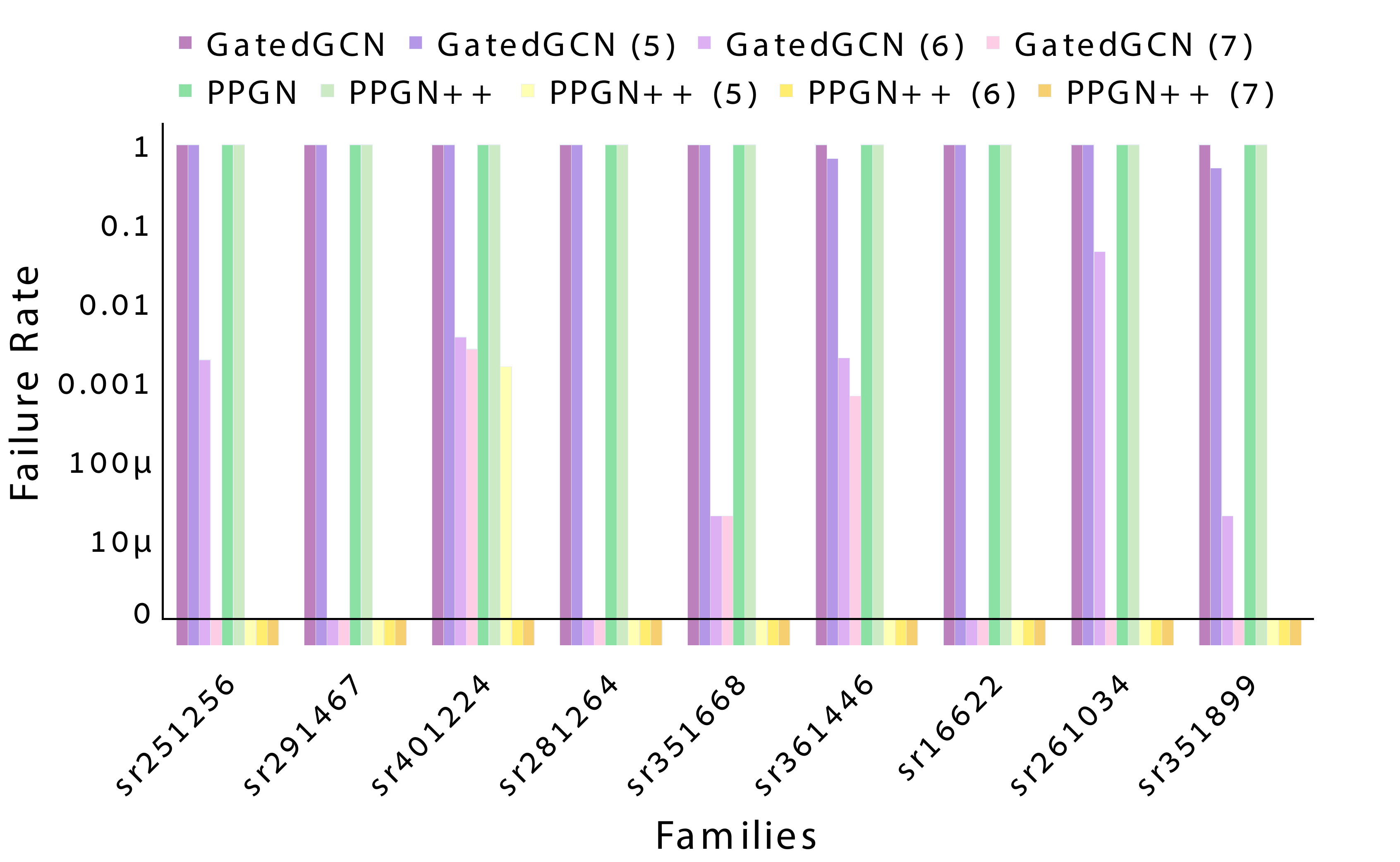}
    \vspace{-20pt}
    \caption{Failure rate (log scale) for distinguishing SR graphs, the lower the better.}
    \label{fig:sr}
    \vspace{-10pt}
\end{figure}
\subsection{Real-World Datasets}
The efficacy of increasing polynomial expressive power on real-world data (molecular graphs datasets) was evaluated on $3$ graph regression tasks: ZINC, ZINC-full and Alchemy.
\textbf{Training.} We followed the training protocol mentioned in \citep{Dwivedi2020benchmarking} for ZINC and ZINC-full and the protocol from \citep{lim2022sign} for Alchemy. All of our trained models obey a $500K$ parameter budget. Further details regarding the training procedure and model parameters can be found in Appendix \ref{s:imp_details}. \\
\textbf{Baselines.} The baseline results for the ZINC $12K$ experiment were obtained from \citep{zhao2022setGNN}, except 
for PPGN and GatedGCN, which we calculated. For ZINC-full and Alchemy we used the results from \citep{lim2022sign}.\\
\textbf{Results.} 
The mean absolute error (MAE) over the test set is reported in Table \ref{tab:zinc_small} for ZINC $12K$ and Table \ref{tab:alchemy+zinc_full} for ZINC-full and alchemy. In both tables, PPGN++ ($6$) achieves SOTA results across all $3$ datasets. In addition, for both test model families there is a clear correlation between higher $d$ (polynomial expressiveness) and test error. Furthermore, PPGN++ ($5$) and PPGN++ ($6$), which produce the top results in all $3$ experiments, are the only $2$ models (including all baselines) which are provably strictly more powerful than 3-WL. Our results add evidence to the guiding hypothesis that increases in expressivity facilitate improved downstream performance.

\vspace{-15pt}
\begin{table}[H]
\renewcommand{\tabcolsep}{1.6pt}
\small
\centering
\caption{Results on ZINC $12K$ \citep{Dwivedi2020benchmarking} dataset. Lower is better, best model is marked in \textbf{bold}.}
\vspace{5pt}
\begin{tabular}{l|c}
\toprule
Model                 & Test MAE                   \\ \hline
GCN \citep{kipf2016gcn}                  & $.321 \pm .009$          \\
GIN \citep{xu2019how}                  & $.163 \pm .003$          \\
PNA \citep{corso2020pna}                  & $.140 \pm .006$          \\
GSN \citep{bouritsas2022improving}                  & $.115 \pm .012$          \\
PF-GNN \citep{dupty2022pfgnn}                & $.122 \pm .010$          \\
GIN-AK \citep{Zhao2021akgnn}               & $.080 \pm .001$          \\
CIN \citep{bodnar2021cell}                  & $.079 \pm .006$          \\
SetGNN \citep{zhao2022setGNN}               & $.075 \pm .003$          \\\hline
GatedGCN \citep{Bresson2017Gated}               & $.265 \pm .015$          \\
GatedGCN ($4$)          & $.150 \pm .005$          \\
GatedGCN ($5$)          & $.138 \pm .003$          \\
GatedGCN ($6$)        & $.106 \pm .003$ \\ \hline
PPGN \citep{maron2019provably}                 & $.079 \pm .005$          \\ 
PPGN++                & $.076 \pm .003$          \\
PPGN++ ($5$)          & $.072 \pm .005$          \\
\textbf{PPGN++ ($6$)} & $\bm{.071 \pm .001}$ \\ \bottomrule
\end{tabular}
\label{tab:zinc_small}
\vspace{-12pt}
\end{table}

\section{Conclusions}
We propose a novel framework for evaluating the expressive power of GNNs by evaluating their ability to approximate equivariant graph polynomials. Our first step was introducing a basis for those  polynomials of any degree. We then utilized Prototypical graph models to determine the computability of these polynomials with practical GNNs. This led to a method for increasing the expressivity of GNNs through the use of precomputed polynomial features, resulting in a significant improvement in empirical performance.

Future research could focus on several promising directions. One direction can reduce the number of features passed into GNNs by working with a generating set of polynomials rather than the complete basis. Additionally, incorporating features on nodes and edges, as outlined in \cref{ss:gen_and_discussion} and \cref{a:attributed_graphs}, could further improve performance. Another exciting avenue for exploration can incorporate otherwise uncomputable polynomials as computational primitives into GNN layers, rather than as features, to increase expressiveness. Finally, it would be beneficial to study Prototypical graph models to identify families with optimal properties related to expressiveness, memory/computation complexity, and the size of the contraction bank. 
\section{Acknowledgments}
OP is supported by a grant from Israel CHE Program for Data Science Research Centers. DL is supported by an NSF Graduate Fellowship. We thank Nicolas Usunier for his insightful remarks and the anonymous reviewers for their helpful comments.

\bibliography{equipoly}
\bibliographystyle{icml2023}

\newpage
\appendix
\onecolumn

\section{Implementation Details}\label{s:imp_details}
\subsection{Datasets}\label{ss:datasets_details}
\paragraph{SR.} The SR dataset \citep{bouritsas2022improving} is composed of $9$ families of Strongly Regular graphs. Each family has a $4$ dimensional representation: $n$ the number of nodes in the graph, $d$ the degree of each node, $\lambda$ the number of mutual neighbors of adjacent nodes and $\mu$ the number of mutual neighbors of non-adjacent nodes. Table \ref{tab:sr_sizes} shows the size of each Strongly Regular family from the dataset.
\begin{table}[h!]
\renewcommand{\tabcolsep}{1.6pt}
\small
\centering
\caption{Sizes of Strongly Regular Families \citep{bouritsas2022improving}}
\vspace{5pt}
\begin{tabular}{l|c|c|c|c|c|c|c|c|c}
\toprule
Familty & (16,6,2,2) & (25,12,5,6) & (26,10,3,4) & (28,12,6,4) & (29,14,6,7) & (35,16,6,8) & (35,18,9,9) & (36,14,4,6) & (40,12,2,4) \\ \hline
Number of Graphs & 2 & 15 & 10 & 4 & 41 & 3854 & 227 & 180 & 28\\
\bottomrule
\end{tabular}
\label{tab:sr_sizes}
\end{table}
\vspace{-15pt}
\paragraph{ZINC.}
The ZINC dataset is a molecular graph dataset composed of $\sim 250K$ molecules. The regression criterion is a molecular property known as the constrained solubility. Each molecule has both node features and edge features. Node features represent the type of heavy atoms ($4$ types) and edge features the type of bonds between them ($28$). The average number of nodes in a graph is $23.15$ and the number of edges is $49.8$.
There are two versions of the dataset used for learning: ZINC $12K$ which has train/val/test split of $10000/1000/1000$ and ZINC-full with a $2200011/24445/5000$ split. Both data splits can be obtained from \citep{Fey2019PyG} 
\paragraph{Alchemy.}
Alchemy is also a molecular graph dataset composoed of $12000$ graphs ($10000/1000/1000$ split taken from \citep{lim2022sign}). The average number of nodes is $10.1$ and the number of edges is $20.9$. The Regression target in this dataset is a $12$-dimensional vector composed of a collection molecular properties : dipole moment, polarizability, HOMO, LUMO, gap, $R^2$, zero point energy, internal energy, internal energy at $298.15 K$, enthalpy at $298.15 K$ , free energy at $298.15 K$ and heat capacity at $298.15 K$. Each graph has node features ($6$-dimensional atom type indicator) and edge features ($4$-dimensional bond type indicator). 
\subsection{Training Protocol}\label{ss:tain_details}
\paragraph{ZINC.}
For the ZINC and ZINC-full experiments we followed the training protocol from \citep{Dwivedi2020benchmarking}. The protocol includes parameter budget ($500K$), predefined $4$ random seeds and a learning rate decay scheme that reduces the rate based on the validation error (factor $0.5$ and patience factor of $10$ epochs). Initial learning rate was set to $0.002$ and training stopped when reached $10^{-5}$. Batch size was set to $128$. Test error at last epoch was reported. When using polynomial features, we removed the polynomials that had no response over the dataset. Namely, let $f:\Real^{n^2}\too \Real^{n^2}$ be an equivariant polynomial and $\mathcal{X}=\set{\tX}$ be a graph dataset. $f$ does not have a response over $\mathcal{X}$ if $\forall\tX\in\mathcal{X}$, $f(\tX)=0$. Similarly to \citep{Barcelo2021LocalGraphParameters} we normalized the additional features to have a unit norm. For PPGN++ we used $1/1$ of the edge based $5$-degree polynomials and $8/11$ of the $6$-degree polynomials. For GatedGCN we used $2/2$ of the node based $3$-degree polynomials, $6/6$ of the $4$-degree polynomials, $23/23$ of the $5$-degree polynomials and $83/85$ of the $6$-degree polynomials. models were trained using the \textsc{LAMB} optimizer \citep{You2019lamb} on a single Nvidia V-$100$ GPU. The models were trained using the PyTorch  framework \citep{Paszke2019pytorch}.
\paragraph{Alchemy.}
We followed the training protocol from \citep{lim2022sign}. The protocol includes averaging results on $5$ random seeds and learning rate decay scheme that reduces the rate based on the validation error (factor $0.5$ and patience factor of $20$ epochs). Initial learning rate was set to $10^{-3}$ and training stopped when reached $10^{-5}$.  Batch size was set to $128$. Test error at last epoch was reported. When using polynomial features, we removed the polynomials that had no response over the dataset and normalized them in the same way as in the ZINC experiment. For PPGN++ we used $1/1$ of the edge based $5$-degree polynomials and $8/11$ of the $6$-degree polynomials.
models were trained using the \textsc{LAMB} optimizer \citep{You2019lamb} on a single Nvidia V-$100$ GPU. The models were trained using the PyTorch  framework \citep{Paszke2019pytorch}.
\subsection{Architectures}\label{ss:arch_details}
\paragraph{GatedGCN.} We used the model as it defined in \citep{Bresson2017Gated} and implemented in \citep{lim2022sign}. 
For the ZINC $12K$ experiment we to used a $16$-layer model (same baseline as used in \citep{lim2022sign}) with feature dimension of size $77$ for the baseline model and $75$ for the models with polynomial features.  
For the SR dataset we used a $4$-layer network with hidden dimension size of $150$.
The polynomial features were added to the initial input node features via concatenation. 
\paragraph{PPGN++.}
The PPGN++ architecture is based on the PPGN architecture \citep{maron2019provably}. The original PPGN layer is defined by the following equation:
\begin{equation*}
    \tZ^{(k+1)} = \texttt{m}_3\brac{ \texttt{m}_1(\tZ^{(k)}) \circledast
 \texttt{m}_2(\tZ^{(k)}), \tZ^{(k)}}
\end{equation*} 
For $\tZ\in\Real^{n^2\times d}$. While this layer definition cannot approximate all $C\in\gB$ from $\gF_e$, it is possible to naively incorporate all the linear and constant basis \citep{maron2018invariant} to obtain full approximation power. As mentioned in Section \ref{ss:gnn_expressive_power} we suggest to add this expressiveness to the layer in a more compact manner:
\begin{equation*}
    \tZ^{(k+1)} = \bar{\texttt{m}}_3\brac{ \bar{\texttt{m}}_1\brac{\tZ^{(k)},\tZ^{(k)T}} \circledast
 \bar{\texttt{m}}_2(\tZ^{(k)}), \tZ^{(k)}}
\end{equation*}
where 
$$\bar{\texttt{m}}_i=\parr{\bar{\texttt{m}}^{\tiny\text{diag}}_i,\bar{\texttt{m}}^{\tiny\text{off-diag}}_i},$$
defines a separate MLP for diagonal elements and off-diagonal elements. In practice, we implement this separation by adding an identity matrix as additional feature before applying an MLP on the tensor's features.

For the ZINC $12K$ experiment we used a $8$-layer network with hidden dimension size of $95$. For ZINC-full and Alchemy we used a $6$-layer network with hidden dimension size of $110$.
We ran parameter search over the number of layers $L\in\set{4,6,8}$ and hidden dimension size $h\in\set{95,110,130}$ while maintaining the $500K$ parameter budget. For the SR experiment we used a $4$-layer network with hidden dimension of size $75$. The polynomial features were added to the initial input features via concatenation.
    
\subsection{Timing}
Table \ref{tab:comp_table} shows a runtime comparison between the preprocessing require to compute polynomial features and training a PPGN++ ($6$) model on the ZINC $12K$ dataset. The time it takes to compute polynomials of degree $7$ is non-negligible and most likely that for higher degrees (or in cases of larger graphs) the runtime will be longer and intractable from some degree. However, for SOTA results which we report in Section \ref{s:experiments} we only use up to $6$ degree polynomial features and the added time used for computing those features is equivalent to only $3$ training epochs. 
Moreover, comparing the running time of other methods puts in perspective the computational time required for computing polynomial features. SetGNN \citep{zhao2022setGNN} reports that the epoch running of their best ZINC model ($0.075$ compared to $0.071$ of PPGN++ ($6$)) is around $25$ seconds. In addition GraphGPS \citep{gps2023recipe}, a state of the art Graph Transformer (test error of $0.07$ on the ZINC dataset), takes $\sim 11.7$ hours to train.  
\begin{table}[h]
\centering
\caption{Runtime comparison on ZINC $12K$: preprocessing vs. training. }
\vspace{5pt}
\begin{tabular}{|l|c|}
\hline
 & Time (Seconds) \\ \hline
finding all $\gF_e$ non-computable polynomials up to degree $7$. & 5 \\ \hline
compute all $5$ degree polynomial features for the entire ZINC $12K$ dataset. & 10 \\ \hline
compute all $6$ degree polynomial features for the entire ZINC $12K$ dataset. & 23 \\ \hline
compute all $7$ degree polynomial features for the entire ZINC $12K$ dataset. & 310 \\ \hline
Average runtime of training PPGN++ ($6$) on ZINC $12K$ & 4110 (15.5 per epoch) \\ \hline
\end{tabular}

\label{tab:comp_table}
\end{table}

\section{Proof of Theorem \ref{thm:poly}.}\label{a:proof_of_poly}

\paragraph{General definitions and setup.}
We denote an input graph data points represented by $\tX\in\Real^{n^2}$. We denote by the vector space of all polynomials $P:\Real^{n^2}\too\Real^{n^2}$ by  $\mathfrak{P}=\Real^{n^2}\otimes \Real[\tX]$, where $\otimes$ is the tensor product and $\Real[\tX]$ denotes the module of polynomials with indeterminate $\tX_{11},\ldots,\tX_{nn}$. The space of polynomials $\mathfrak{P}$ is spanned by the monomial basis
\begin{equation}\label{ea:monomial}
    M(\tX) = \delta^{a,b}\otimes \prod_{r,s=1}^n \tX_{r,s}^{\tA_{r,s}}
\end{equation}
where $\tA\in \Nat_0^{n^2}$, $\Nat_0=\set{0,1,\ldots}$, and $\delta^{a,b} \in \Real^{n^2}$ is a matrix satisfying 
\begin{equation*}
    \delta^{a,b}_{i,j} = \begin{cases} 1 & \text{if } a=i, $b=j$ \\ 0 & \text{o/w} \end{cases}    
\end{equation*}
That is $\delta^{a,b}$, $a,b\in [n]$ is a basis for $\Real^{n^2}$. 

The degree of a polynomial is the maximal degree of its monomials defined by
\begin{equation}
    \deg M(\tX) = \sum_{r,s=1}^n \tA_{r,s}
\end{equation}
We denote by $\mathfrak{P}_d$ the space of all polynomials of degree at most $d$.

\paragraph{Enumerating monomials with multi-graphs $H$.}
Next, we define $H=(V,E,(a,b))$ to be a multi-graph with node set $V=[n]$, and edge multiset defined by the matrix $\tA$, that is $(r,s)$ appears $k\in \Nat_0$ times in $E$ iff $\tA_{r,s}=k$. Lastly $(a,b)$ is the red edge. We can therefore identify monomials with multi-graphs $H$, \ie, 
\begin{equation}
    M_H=M,
\end{equation}
where $M$ is defined in \eqref{ea:monomial}. 

\paragraph{Action of permutations $S_n$ on polynomials.}
We consider the group of permutations $S_n$ that consists of bijections $g:[n]\too[n]$. The action of $S_n$ on a matrix $\tX$ is defined in the standard way in \eqref{e:g}, \ie, 
\begin{equation}\label{ea:action}
    (g\cdot \tX)_{i,j} = \tX_{g^{-1}(i),g^{-1}(j)}
\end{equation}
where the inverse is used to make this a left action. We define $\mathfrak{P}^{S_n}$ to be the space of permutation equivariant polynomials, namely $P\in \mathfrak{P}$ that satisfy
\begin{equation*}
    g\cdot P(\tX) = P(g\cdot \tX)
\end{equation*}
for all $g\in S_n$ and $\tX \in \Real^{n^2}$. A standard method of projecting a polynomial in $\mathfrak{P}$ onto the equivariant polynomials $\mathfrak{P}^{S_n}$ is via the symmetrization (Reynolds) operators:
\begin{equation}
    \bar{P}(\tX) = \sum_{g\in S_n} g\cdot P( g^{-1} \cdot \tX)
\end{equation}
Let us verify that indeed $\bar{P}\in \mathfrak{P}^{S_n}$:
\begin{align*}
        \bar{P}(h\cdot \tX) &= \sum_{g\in S_n} g\cdot P( g^{-1} \cdot (h\cdot \tX)) \\ &= \sum_{g\in S_n} g\cdot P( (h^{-1}g)^{-1} \cdot \tX)
        \\ &= \sum_{g\in S_n} hg\cdot P( g^{-1} \cdot \tX) \\&= h\cdot \bar{P}(\tX)
\end{align*}

\paragraph{Symmetrization of monomials.}
The key part of the proof is computing the symmetrization of the monomial basis $M_H$ via the symmetrization operator:
\begin{align*}
    Q_H(\tX)_{i,j} &= \sum_{g\in S_n} \brac{g\cdot M_H(g^{-1}\cdot \tX)}_{i,j}\\ 
    &=\sum_{g\in S_n} \brac{M_H(g^{-1}\cdot \tX)}_{g^{-1}(i),g^{-1}(j)}\\&=\sum_{g\in S_n} \brac{\delta^{a,b}\otimes \prod_{r,s=1}^n (g^{-1}\cdot \tX)_{r,s}^{\tA_{r,s}}}_{g^{-1}(i),g^{-1}(j)}\\    
    &=\sum_{g\in S_n} \brac{\delta^{a,b}\otimes \prod_{r,s=1}^n \tX_{g(r),g(s)}^{\tA_{r,s}}}_{g^{-1}(i),g^{-1}(j)}\\    
    &=\sum_{g\in S_n} \brac{\delta^{a,b}\otimes \prod_{r,s=1}^n \tX_{r,s}^{\tA_{g^{-1}(r),g^{-1}(s)}}}_{g^{-1}(i),g^{-1}(j)}\\    
    &=\sum_{g\in S_n} \delta^{a,b}_{g^{-1}(i),g^{-1}(j)} \prod_{r,s=1}^n \tX_{r,s}^{\tA_{g^{-1}(r),g^{-1}(s)}}\\    
    &=\sum_{g\in S_n} \delta^{g(a),g(b)}_{i,j} \prod_{r,s=1}^n \tX_{r,s}^{\tA_{g^{-1}(r),g^{-1}(s)}}
    \end{align*}
where in the second and fourth equality we used the action definition (\eqref{ea:action}), in the fifth equality we re-enumerated $(r,s)\in [n]\times [n]$ with $(r',s')=(g(r),g(s))$, and the last equality uses the fact that $a=g^{-1}(i)$ and $b=g^{-1}(j)$ iff $g(a)=i$ and $g(b)=j$. 

Now let us define the action of $S_n$ on the multi-graph $H$, also in a natural manner: $g\cdot H$ is the multi-graph that results from relabeling each node $i\in [n]$ in $H$ as $g(i)\in [n]$. The multi-graph $g\cdot H$ is isomorphic to $H$ and $(r,s)\in E$ with multiplicity $\ell$ iff $(g(r),g(s)) \in g\cdot E$ with multiplicity $\ell$. If we let $\tA$ be the adjacency matrix of $\tH$ then $g\cdot \tA$ (defined via \eqref{ea:action}) is the adjacency of $g\cdot H$, \ie, $(g\cdot \tA)_{i,j}=\tA_{g^{-1}(i),g^{-1}(j)}$. Furthermore, the red edge in $g\cdot H$ is $(g(a),g(b))$. With these definitions, the above equation takes the form     
    \begin{align}\label{ea:Q_H_orbits_H}
    Q_H(\tX)  &= \sum_{g\in S_n} M_{g\cdot H}(\tX)
\end{align}
Equation \ref{ea:Q_H_orbits_H} is the key to the proof. It shows that $Q_H$ is a sum over all monomials corresponding to the orbit of $H$ under node relabeling $g$, therefore, any two isomorphic multi-graphs $H\cong H'$ would correspond to the same equivariant polynomials $Q_H=Q_{H'}$. Differently put, in contrast to $M_H$ that are enumerated by \emph{labeled} multi-graphs $H$, $Q_H$ are enumerated by \emph{non-isomorphic} multi-graphs $H$. Note that if $H$ has isolated nodes (\ie, not touching any edge), these can be discarded without changing $Q_H$, so for degree $d$ polynomials we really just need to consider graphs with $d$ edges and a single red edge with no isolated nodes, so the maximal number of nodes is at most $\min\set{2d+2,n}$. 

We next show that $\set{Q_H}$, corresponding to all non-isomorphic $H$ with up to $d$ edges, is a basis for $\mathfrak{P}_d$. First, we claim it spans $\mathfrak{P}_d$. Indeed, since every polynomial $P\in\mathfrak{P}_d$ can be written as a linear combination of monomials $M_H$, $P=\sum_k c_k M_{H_k}(\tX)$. Now,
\begin{align*}
    P(\tX) &= \bar{P}(\tX) = \sum_k c_k \bar{M}_{H_k}(\tX) = \sum_k c_k Q_{H_k}(\tX)
\end{align*}
where in the first equality we used the fact that the symmetrization operator fixes $P$, \ie, $\bar{P}=P$, and in the second equality the fact that the symmetrization operator is linear. Next, we claim that $\set{Q_H}$ for non-isomorphic $H$ is an independent set. This is true since each $Q_H$ is a sum over the orbit of $H$, $\set{g\cdot H \vert g\in S_n}$, and the orbits are disjoint sets. Therefore, since the set of all monomials, $M_H$, is independent, also $\set{Q_H}$ is independent.

\paragraph{Formula for $Q_H$.} We found that $Q_H$ is a basis for the equivariant graph polynomials $\mathfrak{P}_d$. Let us write down an explicit formula for it next. The $(i_a,i_b)$ entry of $Q_H(\tX)$ takes the form
\begin{align}\nonumber
    Q_H(\tX)_{i_a,i_b} &= \sum_{g\in S_n} \delta^{g(a),g(b)}_{i_a,i_b} \prod_{r,s=1}^n \tX_{g(r)g(s)}^{\tA_{r,s}} \\ \nonumber
    &= \sum_{g\in S_n} \delta^{g(a),g(b)}_{i_a,i_b} \prod_{(r,s)
    \in E} \tX_{g(r)g(s)} \\ \label{ea:Q_H}
    &= \sum_{\substack{j_1\ne \cdots \ne j_m \in [n]\\  j_a=i_a,j_b=i_b}}\prod_{(r,s)\in E} \tX_{j_r,j_s}
\end{align}
where in the third equality we denote $j_1=g(1), j_2=g(2), \ldots, j_m=g(m)$, and $j_1\ne \cdots \ne j_m \in [n]$ stands for all assignments of different indices $j_1,\ldots,j_m\in [n]$.  

Note that $Q_H$ is proved a basis but is still different from $P_H$ in \eqref{e:P_H} in that it does not sum over repeated indices. The fact that allowing repeated indices is still a basis is proved next. This seemingly small change of basis is crucial for our tensor network connection and analysis in the paper.

\paragraph{$P_H$ is a basis. } We now prove that $P_H$ defined in \eqref{e:P_H} is a basis. For convenience we repeat it below:
\begin{equation}\label{ea:P_H}
    P_H(\tX)_{i_a,i_b} = \sum_{\substack{j_1,\ldots,j_m\in [n] \\ j_a=i_a, j_b=i_a}} \prod_{(r,s)\in E} \tX_{j_r,j_s}
\end{equation}

Denote by $\gH_m$ the set of all multigraphs $H=(V,E,(a,b))$ with $\abs{V}\leq m$. Since the cardinality of $\set{P_H}_{H\in \gH_m}$ is at most that of $\set{Q_H}_{H\in \gH_m}$ it is enough to show that $\set{P_H}_{H\in \gH_m}$ spans the same space as $\set{Q_H}_{H\in \gH_m}$.  

The proof follows an induction on $m$. 
For the base $m=1$ consider all multigraphs $H=(V,E,(a,a))$, where $V=\set{a}$. In this case both \eqref{ea:Q_H} and \eqref{ea:P_H} have vacant sums and
\begin{equation*}
    P_H(\tX)_{i_a,i_a} = \prod_{(r,s)\in E} \tX_{i_r,i_s} = Q_H(\tX)_{i_a,i_a},
\end{equation*}
where for all $(r,s)\in E$ we have $r,s \in \set{a}$. 

Now, for $m\geq 2$, assume $$\text{span}\set{P_H}_{H\in\gH_{m-1}}=\text{span}\set{Q_H}_{H\in\gH_{m-1}}$$ and consider an arbitrary $H\in \gH_m \setminus \gH_{m-1}$. 

If $m=2$, and $a\ne b$, then $H=(V,E,(a,b))$, and $V=\set{a,b}$. 
In this case again both \eqref{ea:Q_H} and \eqref{ea:P_H} have vacant sums and
\begin{equation*}
    P_H(\tX)_{i_a,i_b} = \prod_{(r,s)\in E} \tX_{i_r,i_s} = Q_H(\tX)_{i_a,i_b},
\end{equation*}
where for all $(r,s)\in E$ we have $r,s \in \set{a,b}$. 

In all other cases, consider the space of tuples $[n]^m=\set{(j_1,\ldots,j_m)\vert j_i\in [n], i\in [m]}$, and the action of $S_n$ on this collection via $g\cdot(j_1,\ldots,j_m) = (g(j_1),\ldots,g(j_m))$. The orbits, denoted $o_1,\ldots,o_{B}$ correspond to equality patterns of indices, and $B=\mathrm{Bell}(m)$, the Bell number of $m$. By convention we define $o_1$ to be the orbit
\begin{equation*}
    o_1=\brac{(1,2,\ldots,m)}
\end{equation*}
where we use the orbit notation $[(j_1,\ldots, j_m)]=\set{(g(j_1),\ldots,g(j_m))\vert g\in S_n}$. 
Now, we decompose the index set $\set{(j_1,\ldots,j_m)\in [n]^m \vert j_a=i_a, j_b=i_b}$ to disjoint index sets by intersecting it with $o_\ell$, $\ell\in[B]$. Note that some of these index sets may be empty; we let $c_\ell=1$ in case this index set is not empty, and  $c_\ell=0$ otherwise. 
\begin{align*}
&P_H(\tX)_{i_a,i_b} = \sum_{\ell=1}^B\sum_{\substack{(j_1,\ldots,j_m)\in [n]^m \cap o_\ell \\ j_a=i_a, j_b=i_b}} \prod_{(r,s)\in E} \tX_{j_r,j_s}   \\ 
&=\hspace{-15pt}\sum_{\substack{j_1\ne\ldots\ne j_m\\ j_a=i_a, j_b=i_b}}\prod_{(r,s)\in E} \tX_{j_r,j_s} + \sum_{\ell=2}^B \sum_{\substack{(j_1,\ldots,j_m)\in [n]^m \cap o_\ell \\ j_a=i_a, j_b=i_b}}  \prod_{(r,s)\in E} \tX_{j_r,j_s}  \\
&= Q_H(\tX)_{i_a,i_b} + \sum_{\ell=2}^B\ \sum_{\substack{(j_1,\ldots,j_m)\in [n]^m \cap o_\ell \\ j_a=i_a, j_b=i_b}}  \prod_{(r,s)\in E} \tX_{j_r,j_s}
\end{align*}
For $\ell\geq 2$ consider the polynomial 
\begin{equation*}
    \sum_{\substack{(j_1,\ldots,j_m)\in [n]^m \cap o_\ell \\ j_a=i_a, j_b=i_b}}  \prod_{(r,s)\in E} \tX_{j_r,j_s}
\end{equation*}
In case $c_\ell=1$, this polynomial corresponds to $Q_{H_\ell}(\tX)_{i_a,i_b}$, where we denote by $H_\ell=(V_\ell,E_\ell,(a,b))$ the multigraph that results from unifying nodes in $H$ that correspond to equal indices in $o_\ell$. %
We therefore have
\begin{align*}
    P_H(\tX)_{i_a,i_b}&= Q_H(\tX)_{i_a,i_b} + \sum_{\ell=2}^B c_\ell Q_{H_\ell}(\tX)_{i_a,i_b}
\end{align*}
Since for all $\ell\geq 2$ there is at-least one pair of equal indices in $o_k$, $\abs{V_\ell}\leq m-1$. We can therefore use the induction assumption and express these polynomials using polynomials in $\set{P_H}_{H\in\gH_{m-1}}$. This shows that $Q_H$ can be spanned by $\set{P_H}_{H\in \gH_{m}}$. Since $H\in \gH_m\setminus \gH_{m-1}$ was arbitrary this shows that all $Q_H\in \gH_m \setminus \gH_{m-1}$ can be spanned by elements in $\set{P_H}_{H\in \gH_{m}}$. Now using the induction assumption again and the fact that $\gH_{m-1}\subset \gH_m $ we get that $$\mathrm{span}\set{Q_h}_{H\in \gH_m} \subset \mathrm{span}\set{P_h}_{H\in \gH_m}$$
as required.

\section{Proof of Theorem \ref{thm:contraction}}


\begin{reptheorem}{thm:contraction}
    Let $H$ be some multi-graph and $\gF_\gB\in\set{\gF_n,\gF_e}$. Further, let $H'$ be the multi-graph resulting after contracting a single node in $H$ using one or more operations from $\gB$ to $H$. Then, $H$ is $\gF$-computable iff $H'$ is $\gF$-computable. 
\end{reptheorem}
\begin{proof}

We will use Lemma \ref{lem:always_can_contract_verts} and two auxiliary lemmas:

\begin{lemma}\label{lem:gF_only_affects_1_ring}
All the tensor contractions used in $\gF_n$ and $\gF_e$ only affect the 1-ring neighborhood of the contracted node. 
\end{lemma}
\begin{lemma}\label{lem:assumption_implies_k_can_be_reduced}
Assumption (I) for $H_k$ and $H'_k$ imply that the $k$-th node can be contracted from $H'_k$.
\end{lemma}
    
For conciseness we will use $\gF=\gF_\gB$ to denote a graph model. If $H'$ is $\gF$ computable then there exists a sequence of contractions $C_{i_1},C_{i_2},...,C_{i_k}\in \gF$ that contracts $H'$ to the red edge. Then  $C,C_{i_1},C_{i_2},...,C_{i_k}$ is a sequence contracting $H$ to the red edge. Therefore $H$ is computable with $\gF$.

    The other direction is more challenging. We assume $H$ is computable with $\gF$ and need to prove $H'$ is computable with $\gF$. $H=(V=[m],E,(a,b))$ has some sequence of tensor contraction contracting all vertices in $H$ until only $a$ and $b$ are left ($a$ and $b$ could be the same node). 
%
%
 %
    Without losing generality $a\leq b$, and we assume the order of the node contraction from $H$ is $1,2,\ldots,a-1$. We will say that nodes $i,j\in [m]$ are neighbors (in $H$) if they share an edge. 
    
    A key property we use is proved in Lemma \ref{lem:always_can_contract_verts} that shows that using contractions from $\gB$, we can always contract a node if it has at-most $1$ and $2$ neighbors for $\gF_n$ and $\gF_e$, respectively.    
        
    We have that $H'=(V',E',(a,b))$ resulted from $H$ by contracting a single node using contractions in $\gF$. Therefore $|V'|=m-1$, and we let $c$ be the contracted node. Since $c$ is contracted then necessarily $c\ne a,b$. Therefore $c$ belongs to the $H$ contraction series, and in our notation that means $c<a$. Now we will use the series $1,2,\ldots,\bar{c},\ldots,a-1$ (a bar indicates a missing index) as a node contraction series for $H'$. What we need to prove is that this is indeed a series of node contractions that can be implemented with tensor contractions from $\gB$.
    
    To show that we will prove by induction the following claim. 
    We will compare the two node contraction sequences:
    \begin{align*}        
        H:  & \quad 1, 2, \ldots, c-1, c, c+1, \ldots, a-1 \\
        H':  & \quad 1, 2, \ldots, c-1, \emptyset, c+1, \ldots, a-1
    \end{align*}
    where $\emptyset$ means no node contraction done. We enumerate these steps using $k=1,\ldots,a$. We denote by $H_k$ and $H'_k$ the corresponding graphs \emph{before} performing the $k$-th contraction. So $H_1=H$, and $H'_1=H'$. 
    We claim the following holds at the $k$-th step:
    \begin{enumerate}[(I)]    
        \item Any pair of nodes $i,j$ where at-least one node is not $c$ or a neighbor of $c$ satisfy: $i,j$ are neighbors in $H_K$ iff they are neighbors in $H'_k$  
    \end{enumerate}


Where we define the 1-ring of a node $c$ to be the set of nodes that includes: $c$ and all the nodes that share an edge with $c$. Before proving this by induction we note that if the induction hypothesis holds at the $k$-th step then the $k$-th node can be contracted from $H'$ using operations from $\gB$, see Lemma \ref{lem:assumption_implies_k_can_be_reduced}.

    \paragraph{Base case, $k=1$:}
    For $k=1$ we compare the original $H$ and $H'$. Consider two nodes $i,j$ not in the 1-ring of $c$ in $H$. Lemma \ref{lem:gF_only_affects_1_ring} asserts that contraction of a node only affect its immediate neighbors. Therefore any edge/no edge between $i$ and $j$ will be identical in $H$ and $H'$.
   
    \paragraph{Induction step:}
    We assume by the induction assumption that the $H_{k-1},H'_{k-1}$ satisfy (I) and prove it for $H_k,H'_k$.  

    Consider the node $k-1$ that was contracted at the $k-1$ stage. Let $\set{d,e}$ be its neighbor set in $H_{k-1}$ (could be empty, with a single node, or at-most two nodes). There are three cases: (i) $k-1=c$, (ii) $k-1$ is a neighbor of $c$ in $H_{k-1}$ (\ie, $k-1\in\set{d,e}$), and (iii) $k-1$ is not in the 1-ring of $c$ in $H_{k-1}$ (\ie, $k-1\notin \set{c,e,d}$). 
    
    In case (i), its contraction will only affect its 1-ring in $H_{k-1}$ (see Lemma \ref{lem:gF_only_affects_1_ring}), and in $H'_{k-1}$ no contraction will happen. Therefore assumption (I) can be carried to $H_k$ and $H'_{k}$.   
    
    In case (ii), $k-1$ is a neighbor of $c$ in $H_{k-1}$. Since the contraction of $k-1$ only affects its 1-ring in $H_{k-1}$ (according to Lemma \ref{lem:gF_only_affects_1_ring}) and the 1-ring of $k-1$ in $H_{k-1}$ is included, aside of $k-1$, in the 1-ring of $c$ at $H_k$. Therefore the neighboring relations in $H_k$ and $H'_k$ outside the 1-ring of $c$ in $H_k$ do not change. The induction assumption (I) on $H_{k-1}$ and $H'_{k-1}$ now implies the assumption holds for $H_k$ and $H'_k$ .

    In case (iii), $k-1$ is not in the 1-ring of $c$ in $H_{k-1}$. Then Lemma \ref{lem:gF_only_affects_1_ring} implies that the 1-ring of $c$ in $H_k$ will not change and the neighborhood changes in the 1-ring of $k-1$ will be identical to $H_k$ and $H'_k$ due to induction assumption (I). 
\end{proof}

\begin{replemma}{lem:always_can_contract_verts}
    $\gF_n$ (for simple graphs) and $\gF_e$ (for general graphs) can always contract a node in $H$ iff its number of neighbors is at-most $1$ and $2$, respectively. 
\end{replemma}
 \begin{proof}
    We start with $\gF_n$: We will use contraction notations from the contraction banks presented in Figure \ref{fig:ideal_graph_models}, left.
    For simple graphs $H$ is simple (see Section \ref{ss:P_H}), and therefore does not have parallel edges. Applications of contractions from the bank of $\gF_n$ cannot introduce parallel edges and therefore any two neighbors in the graph will share a single edge. Furthermore, using $C_2$ we can always reduce the number of self-loops generated during the tensor computation path to $1$. Lastly, any node, with or without a single self-loop, and with at-most $1$ neighbor is connected to it with at-most a single edge, and therefore $C_3$ or $C_4$ will be able to contract it. 

    For $\gF_e$: We will use contraction notations from the contraction bank in Figure \ref{fig:ideal_graph_models}, right. First note that any number of self-loops and parallel edges can be reduced to $1$ using $C_2$ and $C_6$, respectively. Now if a node with a single self-loop has no neighbors in $H$ then is can be contracted with $C_1$. Now in the case a node $i$ has $1$ or $2$ neighbors in $H$ we can cancel its self-loop (if it has one) as follows. Let $j$ denote one of its neighbors. Then we first apply $C_5$ between $i$ (top) and $j$ (bottom), then we apply $C_6$ if necessary to make the existing edge between $i,j$ directing towards $j$, and lastly apply $C_6$ to have a single edge between $i$ and $j$. 

    Lets recap: we have now a node $i$, without self-loops, with a single edge going to its $1$ or $2$ neighbors. We can change the direction of these edges by applying $C_6$, if required. Now, we can use $C_3$ or $C_7$ to contract node $i$.     

    In the other direction if the number of neighbors is greater than 1 for $\gF_n$ and 2 for $\gF_e$ then inspection of the respective contraction banks shows that edges cannot be completely removed between nodes without contraction and no contraction operators for nodes with valence 2 and 3 exists for $\gF_n$ and $\gF_e$, respectively. 
    \end{proof}

    \begin{replemma}{lem:gF_only_affects_1_ring}
All the tensor contractions used in $\gF_n$ and $\gF_e$ only affect the 1-ring neighborhood of the contracted node.
\end{replemma}
\begin{proof}
Inspection of the tensor contraction banks of $\gF_n$ and $\gF_e$ (see Figure \ref{fig:ideal_graph_models}, contracted nodes are in gray) shows that any contraction of a node can introduce new edges in its 1-ring but does not affect neighboring relation outside the 1-ring.  
\end{proof}

\begin{replemma}{lem:assumption_implies_k_can_be_reduced}
Assumption (I) for $H_k$ and $H'_k$ imply that the $k$-th node can be contracted from $H'_k$.
\end{replemma}
\begin{proof}
Indeed, there are 3 options for the $k$-th node:  (i) $k=c$, (ii) $k$ is a neighbor of $c$ in $H_{k}$, and (iii) $k$ is not in the 1-ring of $c$ in $H_k$.     
    
Since $k$ can be contracted from $H_k$ by definition, Lemma \ref{lem:always_can_contract_verts} imply that we only need to show that $k$ has at-most the same number of neighbors in $H'_k$ in order to prove the lemma. We show that next. 

In case (i): since $k=c$ no contraction is to take place in $H'_k$. In case (ii): hypothesis (I) imply that the number of neighbors of $k$ in $H'_k$ is at most that in $H_k$. In case (iii): Hypothesis (I) implies that $k$ has the same neighbors in $H_k$ and $H'_k$.  
\end{proof}

\section{Proof of Theorem \ref{thm:gnn_poly_lower_bound}}\label{s:approximation_proof}

To prove Theorem \ref{thm:gnn_poly_lower_bound} it is enough to show that MPNN and PPGN++ can approximate any polynomial computable by the matching Prototypical models, namely node based $\gF_n$ and edge based $\gF_e$. We show that in the following theorem:
\begin{theorem}\label{thm:gnn_approximate}
    For any compact input domain,
    PPGN++ and MPNN can arbitrarily approximate any polynomial computable by the Prototypical graph models $\gF_e$ and $\gF_n$, respectively.
\end{theorem}

\begin{proof}
    The proof of this theorem has two parts. First, we will show that for $\eps > 0$, a single layer of PPGN++ and MPNN  can approximate any contraction operation $C\in \gB$ of the corresponding graph model. The second part will show that a composition of those layers can approximate any finite sequence $f\in\gF_{\gB}$, i.e any polynomial computable by $\gF_{\gB}$. 
    

    \paragraph{Part I.} Let $\Omega\subset \Real^{n\times d}$ be an arbitrary compact set and $\tY_i\in\Real^n$ be the ith column of $\tY\in\Omega$. Consider $C\in \gB_n$, the tensor contraction bank of $\gF_n$. We will show that a single layer of MPNN (eq \ref{eq:MPNN}) can approximate $C$. To do so we will write the operations explicitly and verify that a MPNN layer can approximate them: 
    \begin{itemize}
        \item $C_1\too\tY^{(k+1)}_j=\one\one^T \tY^{(k)}_i$.

        \item $C_2\too\tY^{(k+1)}_l=\tY_i^{(k)}\odot \tY_j^{(k)}$ where $\odot$ is element-wise product.
        \item $C_3\too\tY^{(k+1)}_j=\tX\one$. 
        \item $C_4\too\tY^{(k+1)}_j=\tX\tY^{(k)}_i$.
    \end{itemize}
    In order for MPNN approximate the mentioned functions we need to argue that $\texttt{m}$ can approximate several functions. We assume that simple functions, such as constant functions and feature retrieving, can be computed exactly for $\texttt{m}$. To justify approximation of element-wise product we use the universal approximation theorem \citep{hornik1991universal}:
    \begin{theorem}
        The set of one hidden layer MLPs with a continuous $\sigma$, i.e, $M(\sigma)=\text{span}\set{\sigma(w^Tx+b)|w\in\Real^n, b\in\Real}$ is dense in $C(\Real^n)$ in the topology of uniform convergence over compact sets if and only if $\sigma$ is not a polynomial.
    \end{theorem}

    In the next case we set $\Omega\subset \Real^{n^2\times d}$ be an arbitrary compact set ($\tZ_i\in\Real^{n^2}$ for $\tZ\in\Omega$).
    We repeat the same process for PPGN++ (eq \ref{eq:PPGN++}) and $C\in\gB$ of $\gF_e$:   
    \begin{itemize}
        \item $C_1\too\tZ_j^{(k+1)}=\one\one^T(\diag(\tZ^{(k)}_i)\one\one^T)$. 
        \item $C_2\too\tZ_l^{(k+1)}=\diag(\tZ_i^{(k)}\cdot\tZ_j^{(k)})$.
        \item $C_3\too\tZ_j^{(k+1)}=\diag(\one\one^T\tZ_i^{(k)})$. 
        \item
        $C_4\too\tZ_j^{(k+1)}=\diag(\tZ_i^{(k)})\one\one^T$. 
        \item $C_5\too\tZ_l^{(k+1)}=\tZ_i^{(k)}\cdot\tZ_j^{(k)}$.
        \item $C_6\too\tZ_j^{(k+1)}=\tZ_i^{(k)T}$.
        \item $C_7\too\tZ_l^{(k+1)}=\tZ_i^{(k)}\tZ_j^{(k)}$.
    \end{itemize}
    Here the only addition is the diag function, which given a matrix returns a diagonal matrix with the matrix diagonal. This could computed using $\bar{\texttt{m}_i}$ since it computes different functions for the diagonal and off-diagonal elements.

    \paragraph{Part II.}
    This part of the proof will show that any sequence $f\in\gF_e$ (or $\in\gF_n$), namely a polynomial computable by this graph model, can be approximated by a composition of PPGN++ (or MPNN) layers. To prove that we can use Lemma 6 from \citep{lim2022sign} that states the following:
    \begin{lemma} [Layer-wise universality implies universality]
    Let $\gZ\subseteq \Real^{d_0}$ be a compact domain, let $\gF_1,\cdots,\gF_L$ be a families of continuous functinos where $\gF_i$ consists of functions from $\Real^{d_{i-1}}\too\Real^{d_i}$ for some $d_1,\cdots,d_L$. Let $\gF$ be a family of functions $\set{f_L\circ\cdots\circ f_1: \gZ\too\Real^{d_L},f_i\in\gF_i}$ that are compositions of functions $f_i\in\gF_i$.

    For each $i$, let $\Phi_i$, be a family of continuous functions that universally approximates $\gF_i$. Then the family of compositions $\set{\phi_L\circ\cdots\circ\phi_1:\phi_i\in\Phi_i}$ universally approximates $\gF$.       
    \end{lemma}

    Based on the first part proof, using this lemma while pluging in $\gB$ as $\gF_i$ for every $i$ and the corresponding GNN layer as $\Phi_i$ (also for every $i$) shows that PPGN++ and MPNN are universal approximators for $\gF_e$ and $\gF_n$, respectively.

\end{proof}


\begin{figure}[t]
    \centering
    \includegraphics{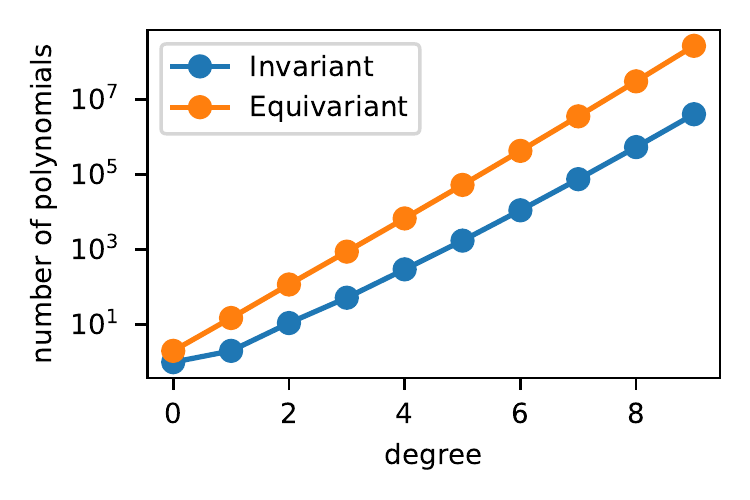}
    \caption{The number of graph invariant and equivariant polynomials scales exponentially with the degree of the polynomial. Here, we count the number of polynomials in the asymptotic limit of graphs of $n \to \infty$ nodes. Graphs with a fixed number of nodes will have fewer invariant polynomials. See \Cref{app:counting_section} for further details on the generating functions and explicit counts. }
    \label{fig:counts}
\end{figure}

\section{Proof of Proposition \ref{thm:ppgn_poly_lower_bound}}\label{s:ppgn_approximation_proof}
Let $f:\Real^{n^2\times d}\too \Real^{n^2\times d'}$ be a PPGN block, defined by the following equation (as portrayed in \citep{maron2019provably}):
\begin{equation}
    \tZ^{(k+1)} = \bar{\texttt{m}}_3\brac{ \bar{\texttt{m}}_1(\tZ^{(k)}) \circledast
 \bar{\texttt{m}}_2(\tZ^{(k)}), \tZ^{(k)}}.
    \label{eq:PPGN}
\end{equation}
Also, let
\begin{equation*}  
\tZ=\begin{bmatrix}
    a & a \\
    b & b \\
\end{bmatrix}\; a,b\in\Real
\end{equation*}
to be an input tensor. A PPGN network (composition of PPGN blocks) cannot approximate the transpose operator $P_H(\tZ)=\tZ^T$ due to the fact that the PPGN block maintains the row structure of $\tZ$, i.e
\begin{equation*}  
f(\tZ)_i=\begin{bmatrix}
    x & x \\
    y & y \\
\end{bmatrix}\; x,y\in\Real,
\end{equation*} 
when $f(\tZ)_i\in \Real^{n^2}$ (for $i\in\brac{d'}$) denotes the ith slice of $f(\tZ)$ along the last dimension.
This holds since each PPGN block is composed of Siamese element-wise operations and matrix multiplications which preserve this structure.
a simple induction can generalize this claim for a composition of blocks while the extension for larger size graphs is also trivial.

\section{Equivariant Polynomials of Attributed Graphs}
\label{a:attributed_graphs}
Repeating the derivations in \Cref{a:proof_of_poly} for the case of equivariant polynomials of attributed graphs, $P:\Real^{n^2\times f}\too\Real^{n^2}$, suggests that we should enumerate each $Q_H$ and $P_H$ with a multi-graph $H=(V,E,(a,b))$, where there are also $f$ types of edges \emph{types} (\eg, colors); we denote by $(r,s;k)\in E$ an edge $(r,s)$ with type $k\in [f]$. This gives the formulas 
\begin{align}
Q_H(\tX)_{i_{a},i_b} &= \sum_{\substack{j_1 \ne \ldots  \ne j_m \in [n]\\ j_a=i_a, j_b=i_b}}\prod_{(r,s;k)\in E} \tX_{j_r,j_s,k} \\
    P_H(\tX)_{i_{a},i_b} &= \sum_{\substack{j_1,\ldots j_m \in [n]\\ j_a=i_a, j_b=i_b}}\prod_{(r,s;k)\in E} \tX_{j_r,j_s,k}
\end{align}
where the degree of $P_H,Q_H$ is the total number of edges of all types, counting multiplicities. The basis $\set{Q_H}$ (and consequently $\set{P_H}$) for equivariant polynomials in this case is achieved by considering all non-isomorphic $H$ (comparing both edge types and multiplicity) with total number of edges up to $f$.   



Equivariant maps $\mathbb{R}^{n^2 \times f} \to \mathbb{R}^{n^2}$ are isomorphically equivalent to the module $(\Real[\mathbb{R}^{n^2 \times f}] \times \mathbb{R}^{n^2})^{S_n}$ , where $\Real[\mathbb{R}^{n^2 \times f}]$ is the polynomial ring on the vector space $\mathbb{R}^{n^2 \times f}$. Similarly to the proof in \Cref{a:proof_of_poly}, via the Reynolds operator applied to monomials in variables $\tX_{ijk}$, we obtain orbits which as before, correspond to unique subgraphs. For a given monomial, if the variable $\tX_{ijk}$ is contained in that monomial, then we add an edge between $i$ to $j$ and color that edge according to the index $k$. The $k$-index remains invariant under the group operation and is not permuted by the group action. 

To continue the graphical notation, we label this basis by labeling its orbits according to the monomials that appear. We pick a given monomial in the orbit and then for each variable in that monomial, we add an edge with the appropriate color. As before, the equivariant output dimension is colored red, but we make such an edge dotted to more clearly differentiate it with other edges. 

In this expanded graphical notation, we provide two examples below:

The proof that the above forms a basis follows directly from the proof in \Cref{a:proof_of_poly}. We follow the basic steps below.

We denote by the vector space of all polynomials $P:\Real^{n^2 \times f}\too\Real^{n^2}$ by  $\mathfrak{P}=\Real^{n^2 \times f}\otimes \Real[\tX]$, where $\otimes$ is the tensor product and $\Real[\tX]$ denotes the module of polynomials with indeterminate $\tX_{11},\ldots,\tX_{nn}$. The space of polynomials $\mathfrak{P}$ is now spanned by the expanded monomial basis
\begin{equation}
    M(\tX) = \delta^{a,b}\otimes \prod_{r,s=1}^n \prod_{k=1}^f \tX_{r,s,k}^{\tA_{r,s,k}}
\end{equation}
where $\tA\in \Nat_0^{n^2 \times f}$, $\Nat_0=\set{0,1,\ldots}$, and $\delta^{a,b} \in \Real^{n^2}$ is a matrix satisfying 
\begin{equation*}
    \delta^{a,b}_{i,j} = \begin{cases} 1 & \text{if } a=i, b=j \\ 0 & \text{o/w} \end{cases}    
\end{equation*}

Permuations now only act on the first two indices, \ie, 
\begin{equation}\label{ea:action}
    (g\cdot \tX)_{i,j,k} = \tX_{g^{-1}(i),g^{-1}(j),k}.
\end{equation}

Given the last index is invariant to permutations, symmetrization of monomials continues as before where we add the feature dimension:
\begin{align*}
    Q_H(\tX)_{i,j} &= \sum_{g\in S_n} \brac{g\cdot M_H(g^{-1}\cdot \tX)}_{i,j}\\ 
    &=\sum_{g\in S_n} \delta^{g(a),g(b)}_{i,j} \prod_{r,s=1}^n \prod_{k=1}^f  \tX_{r,s,k}^{\tA_{g^{-1}(r),g^{-1}(s),k}}.
\end{align*}

$Q_H$ is a sum over all monomials corresponding to the orbit of $H$ under node relabeling $g$. This forms an equivalence class over orbits of $H$ (i.e., two graphs $H$ and $H'$ are in the same class if they can be obtained from one another via permutations). Since every polynomial is a sum over monomials, symmetrization over these monomials implies each symmetrization falls into one of these orbits. Therefore, similar to the proof as before, the multigraphs $H$ compose the set of equivariant polynomials.

\subsection{Equivariant set polynomials}
As a note, we show here via an example how to form set polynomials from the structure described before. In correspondence with the equivariant polynomials on sets ($\mathbb{R}^{n \times d} \to \mathbb{R}^n$), \citet{segol2019universal} proved any polynomial in this setting takes the following form:
\begin{theorem}[Theorem 2 of \cite{segol2019universal}, paraphrased]
    Any equivariant map on sets $\mathbb{R}^{n \times d} \to \mathbb{R}^n$ can be generated by polynomials of the form 
    \begin{equation}
        P(\tX) = \sum_{|\bm \alpha| \leq n} \vb_{\bm \alpha} q_{\bm \alpha, j}(s_1, \dots, s_t),
    \end{equation}
    where $\vb_{\bm \alpha} = [\vx_1^{\bm \alpha}, \dots, \vx_n^{\bm \alpha}]^\top$, $s_j(\tX) = \sum_{i=1}^n \vx_i^{\bm \alpha_j}$ are the power sum symmetric polynomials indexed by $j \in \left[ {n+d \choose d} \right]$ possible such polynomials up to degree $d$, and $q_{\bm \alpha, j}(s_1, \dots, s_t)$ is a polynomial in its power sum polynomial inputs.
\end{theorem}

In our graphical language, set polynomials correspond to graphs with only multi-edges that are self loops. To recover the above theorem in our graphical notation, we consider each element in the sum above. We identify a given $\vb_{\bm \alpha}$ with the self loops on the equivariant edge with the dotted line. The polynomial $q_{\bm \alpha, j}$ is identified with the polynomial on the rest of the nodes, e.g. let us consider the below graph.

\begin{equation}
    \begin{tikzpicture}[main/.style = {draw, circle},baseline=5pt,semithick] 
\node[main] (1) [] {};
\node[main] (2) [right of=1] {};
\node[main] (3) [right of=2] {};
\path (1) edge [densely dotted, loop above, color=red] (1);
\path (1) edge [loop below, color=blue] (1);
\path (2) edge [loop below, color=orange] (2);
\path (3) edge [loop below, color=green] (3);
\end{tikzpicture} 
\end{equation}
As before, for colors indexed by index zero (orange), index one (green), and index two (blue), the above corresponds to the polynomial
\begin{equation}
    P(\tX) = \begin{bmatrix} \tX_{1,2} \\ \tX_{2,2} \\ \vdots \\ \tX_{n,2}
    \end{bmatrix}  \cdot \left( \sum_{i,j=1}^n \tX_{i,0} \tX_{j,1} \right).
\end{equation}
By inspection, one can see that the above is of the form as stated in \cite{segol2019universal}, i.e. choose $q_{ \bm \alpha,j} = (\sum_{i=1}^n \vx_i^{[1,0,0]})(\sum_{i=1}^n \vx_i^{[0,1,0]}) $ only for $\bm \alpha = [0,0,1]$ and $q_{ \bm \alpha,j} =  0$ otherwise.

\section{Relationship between Equivariant Polynomials,  Homomorphisms, and Subgraph Counting}\label{appendix:homomorphism}

There is a close correspondence between homomorphism counts, subgraph counts, and the evaluation of our polynomials on binary graphs $\tX$, meaning directed graphs with self-loops but no multiedges, i.e. those graphs with adjacency in $ \{0, 1\}^{n \times n}$.
For binary graphs $H$ and $\tX$, a homomorphism is a function $\varphi: V(H) \to V(\tX)$ such that if $(r,s) \in E(H)$, then $(\varphi(r), \varphi(s)) \in E(\tX)$. An isomorphism is a bijective homomorphism whose inverse is also a homomorphism.
We let $\hom(H, \tX)$ denote the number of homomorphisms from $H$ to $\tX$. We let $\inj(H, \tX)$ denote the number of injective homomorphisms from $H$ to $\tX$. The injective homomorphism number is closely related to subgraph counts. If we let $\mathrm{count}(H, \tX)$ denote the number of subgraphs isomorphic to $H$, and let $\Aut(H)$ denote the automorphism group of $H$ (i.e. the set of isomorphisms from $H$ to $H$), then $\inj(H, \tX) = |\Aut(H)|\cdot \mathrm{count}(H, \tX)$. The $|\Aut(H)|$ term is due to overcounting when $H$ has symmetries.

\subsection{Invariant Polynomials and Standard Homomorphisms}

\textbf{$P_H$ and standard homomorphisms.} Let $H$ and $\tX$ be binary graphs. Then it can be seen that the homomorphism count $\hom(H, \tX)$ can be written as~\citep{lovasz2012large}:
\begin{equation}
    \hom(H, \tX) = \sum_{\varphi: V(H) \to V(\tX)} \  \prod_{(r,s) \in E(H)} \tX_{\varphi(r), \varphi(s)},
\end{equation}
where the sum ranges over all functions from $V(H)$ to $V(\tX)$. Choose an ordering $1, \ldots, m$ of the nodes of $V(H)$ and $1, \ldots, n$ of the nodes of $\tX$; writing $\varphi(l) = j_l$, we see that $\hom(H, \tX)$ is exactly equivalent to our (invariant) polynomial basis element $P_H$:
\begin{equation}
    P_H(\tX) = \hom(H, \tX) = \sum_{j_1, \ldots, j_m \in [n]} \prod_{(r,s) \in E(H)} \tX_{j_r, j_s}.
\end{equation}

If $H$ has multiple edges, then let $\tilde H$ be the graph $H$ with any multiple edges reduced to a single edge. If $\tX$ is still a binary graph, it is easy to see that $P_H(\tX) = P_{\tilde H}(\tX)$, so in this case $P_H(\tX) = \hom(\tilde H, \tX)$.

\textbf{$Q_H$ and injective homomorphisms / subgraph counts.} Once again, let $H$ and $\tX$ be binary graphs. The injective homomorphism number can be written in a similar form to the homomorphism count~\citep{lovasz2012large}:
\begin{equation}
    \inj(H, \tX) =  \sum_{\substack{\varphi: V(H) \to V(\tX) \\ \varphi \text{ injective}}} \ \prod_{(r,s) \in E(H)}  \tX_{\varphi(r), \varphi(s)}.
\end{equation}
In this case, the sum ranges over all injective functions $\varphi: V(H) \to V(\tX)$. As in the non-injective case, we write $j_l = \varphi(l)$ for each $l=1, \ldots, m$. By injectivity $j_1 \neq \ldots \neq j_m$. Thus, we have a corrspondence between $\inj(H, \tX)$ and our invariant basis polynomial $Q_H$:
\begin{equation}
    Q_H(\tX) = \inj(H, \tX) =  \sum_{j_1 \neq \ldots \neq j_m \in [n]} \ \prod_{(r,s) \in E(H)}  \tX_{\varphi(r), \varphi(s)} = 
 |\Aut(H)| \cdot \mathrm{count}(H, \tX).
\end{equation}

\subsection{Equivariant Polynomials and Homomorphism Tensors}

\newcommand{\Hom}{\mathrm{Hom}}
\newcommand{\Inj}{\mathrm{Inj}}

Here, we consider our equivariant polynomials basis elements $P_H, Q_H: \Real^{n^2} \to \Real^{n^2}$. On binary graphs, this will also correspond to counts of homomorphisms, except now we have to restrict the homomorphisms to preserve the red edge in an equivariant way. Let $(a,b)$ be the red edge of $H$, and consider any two nodes (possibly the same) $i_a, i_b$ of $\tX$. Then we may define a tensor of homomorphism counts $\Hom(H, \tX) \in \Real^{n^2}$, where $\Hom(H, \tX)_{i_a, i_b}$ is the number of homomorphisms $\varphi: V(H) \to V(\tX)$ such that $\varphi(a) = i_a$ and $\varphi(b) = i_b$. We define a tensor $\Inj(H, \tX)$ of injective homomorphism  counts similarly.

Following the arguments for the invariant case, it is easy to see that
\begin{align}
    P_H(\tX)_{i_a, i_b} & = \Hom(H, \tX)_{i_a, i_b}\\
    Q_H(\tX)_{i_a, i_b} & =
    \Inj(H, \tX)_{i_a, i_b}.
\end{align}
Thus, this gives an interpretation of $P_H(\tX)_{i_a, i_b}$ and $Q_H(\tX)_{i_a, i_b}$, when $H$ and $\tX$ are binary graphs. We expand on this interpretation for $Q_H$ in the next section.

\subsection{$Q_H$ as Subgraph Counts}

\begin{figure}[h!]
    \centering  
    \begin{tabular}{ccc}
    $H$\ \begin{tikzpicture}[main/.style = {draw, circle},baseline=5pt,semithick] 
\node[main] (1) [] {};
\node[main] (2) [above right of=1] {};
\node[main] (3) [below right of=2] {};
\draw[<-] (1) edge (2);
\draw[<-, bend right=15] (3) edge (2);
\draw[<-, bend left=15] (3) edge (2);
\draw[<-, bend left=15] (3) edge (1);
\draw[<-, bend left=15] (3) edge (1);
\draw[<-, bend left=15] (1) edge (3);
\path (2) edge [loop above, color=red] (2);
\end{tikzpicture} 
        &  $\tilde H$
\begin{tikzpicture}[main/.style = {draw, circle},baseline=5pt,semithick] 
\node[main] (1) [] {};
\node[main] (2) [above right of=1] {};
\node[main] (3) [below right of=2] {};
\draw[<-] (1) edge (2);
\draw[<-] (3) edge (2);
\draw[<-, bend left=15] (3) edge (1);
\draw[<-, bend left=15] (3) edge (1);
\draw[<-, bend left=15] (1) edge (3);
\path (2) edge [loop above, color=red] (2);
\end{tikzpicture} \end{tabular} 
    \\[15pt]   
    \begin{tabular}{cc}
    $\tX$
\begin{tikzpicture}[main/.style = {draw, circle},baseline=5pt, semithick] 
\node[main, label={\small $1$}] (1) [] {};
\node[main, label={\small $2$}] (2) [right of=1] {};
\node[main, label=right:{\small $3$}] (3) [below of=2] {};
\node[main, label={\small $4$}] (4) [right of=2] {};
\node[main, label={\small $5$}] (5) [left of=3] {};
\draw[<-, bend right=15] (2) edge (1);
\draw[<-, bend right=15] (1) edge (2);
\draw[<-] (1) edge (3);
\draw[<-] (2) edge (3);
\draw[<-] (4) edge (3);
\draw[<-, bend right=15] (2) edge (4);
\draw[<-, bend right=15] (4) edge (2);
\draw[<-] (3) edge (5);
\end{tikzpicture} 
& 
\begin{tabular}{l}
$|\Aut(\tilde H)| = 2$\\
$Q_H(\tX)_{i,i} = \sum_{j \neq k \in [n] \setminus \{i\}} \tX_{i, j} \tX_{i, k}^2 \tX_{j,k}^2$ \\
    $Q_H(\tX)_{3, 3} = 4$\\
    $Q_H(\tX)_{i, i} = 0$, $i \neq 3$\\
\end{tabular}
    \end{tabular}
\caption{Relationship between $Q_H$ and subgraph counts. The directed multigraph $H$ (top left) has its multiedges turned to single edges to form $\tilde H$ (top right). The size of the automorphism group $\Aut(\tilde H)$ is 2, as we may swap the bottom left and bottom right node while preserving the graph structure. $Q_H(\tX)_{3,3} = 4$ because node 3 participates in 2 subgraphs isomorphic to $\tilde H$, and this subgraph count is scaled by $|\Aut(\tilde H)| = 2$ to get $Q_H(\tX)_{3,3}$.}
    \label{fig:counting_example}
\end{figure}
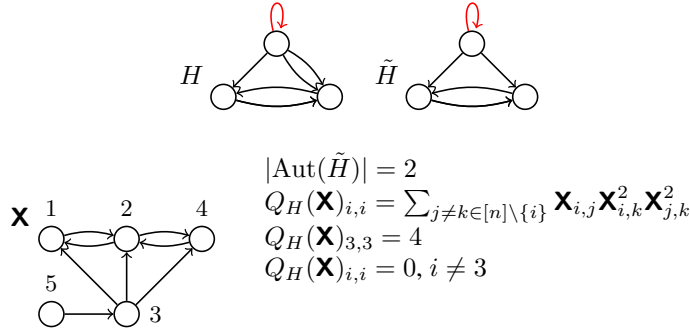

The basis $Q_H$ can be interpreted as subgraph counts when the input is a simple binary graph $\tX \in \{0, 1\}^{n^2}$. For any directed multigraph $H$ with red edge $(a,b)$, let $\tilde H$ denote the same graph $H$ but where any multiple black edges are collapsed to just one black edge; the difference between $H$ and $\tilde H$ is that if $H$ has more than one black edge from node $i$ to $j$, then $\tilde H$ only has one edge.  Then we have that
\begin{equation}\label{eq:ph_count}
    Q_H(\tX)_{i_a, i_b} = |\Aut(\tilde H)| \cdot \scount(\tilde H, \tX, (i_a, i_b)),
\end{equation}
where $\scount(\tilde H, \tX, (i_a, i_b))$ is the number of subgraphs of $\tX$ that are isomorphic to $\tilde H$, after adding a red edge $(i_a, i_b)$ to $\tX$ and labelling edge $(a,b)$ in $\tilde H$ as red. $|\Aut(\tilde H)|$ is the size of the automorphism group of $\tilde H$, which contains the automorphisms $\varphi: V(H) \to V(H)$ that have $\varphi(a) = a$ and $\varphi(b) = b$.

Now, suppose $H$ has a red self loop $(a,a)$, and let $i$ be any node in $\tX$. We have that $Q_H(\tX)_{i_a, i_a}$ is equal to $|\Aut(\tilde H)|$ multiplied by the number of subgraphs of $\tX$ that are isomorphic to $\tilde H$, where the isomorphism maps $i_a$ in $\tX$ to $a$ (the node with the self loop in $H$). Intuitively, this is proportional to the number of subgraphs isomorphic to $\tilde H$ that $i_a$ participates in as the designated red-self-loop node. See Figure~\ref{fig:counting_example} for an illustration.

\textbf{Derivation.} Here, we derive the relationship between $Q_H(\tX)$ and subgraph counts. First, we write out some definitions more precisely. Let $\Aut(\tilde H)$ denote the automorphism group of $\tilde H$. This is the set of permutations $\sigma: V(\tilde H) \to V(\tilde H)$ such that $\sigma(a) = a$, $\sigma(b) = b$, and $(r,s) \in E(\tilde H)$ if and only if $(\sigma(r), \sigma(s)) \in E(\tilde H)$. Further, let $\scount(\tilde H, \tX, (i_a, i_b))$ denote the number of subgraphs of $\tX$ isomorphic to $\tilde H$, where the isomorphism maps $a$ to $i_a$ and $b$ to $i_b$. In other words, it is the number of choices $(V', E')$ such that $V' \subseteq V(\tX)$ and $E' \subseteq E(\tX) \cap (V' \times V')$ where there exists a bijection $\varphi: V(\tilde H) \to V'$ such that $\varphi(a) = i_a$, $\varphi(b) = i_b$, and $(r,s) \in E(\tilde H)$ if and only if $(\varphi(r), \varphi(s)) \in E'$. We call any such map $\varphi$ a subgraph isomorphism, and we may also view it as an injective function $V(\tilde H) \to V(\tX)$.

We will use the fact that for any subgraph isomorphism $\varphi$ and any automorphism $\sigma \in \Aut(\tilde H)$, the map $\varphi \circ \sigma : V(\tilde H) \to V(\tX)$ is also a subgraph isomorphism. Let $(V', E')$ be the subgraph that $\varphi$ maps to. To see that $\varphi$ is a subgraph isomorphism, first note that $\varphi \circ \sigma$ is injective, $\varphi \circ \sigma(a) = \varphi(a) = i_a$, and $\varphi \circ \sigma(b) = \varphi(b) = i_b$. To see the edge preserving property, note that $(r,s) \in E(\tilde H)$ if and only if $(\sigma(r), \sigma(s)) \in E(\tilde H)$ since $\sigma \in \Aut(\tilde H)$. Moreover, $(\sigma(r), \sigma(s)) \in E(\tilde H)$ if and only if $(\varphi(\sigma(r)), \varphi(\sigma(s))) \in E'$ since $\varphi$ is a subgraph isomorphism. Thus, $\varphi \circ \sigma$ is a subgraph isomorphism.

We now derive the formula for $Q_H(\tX)$ in terms of subgraph counts:
\begin{proposition}
    If $\tX \in \{0, 1\}^{n^2}$, then
    \begin{equation}
        Q_H(\tX)_{i_a, i_b} = |\Aut(\tilde H)| \cdot \scount(\tilde H, \tX, (i_a, i_b))
    \end{equation}
\end{proposition}
\begin{proof}
    Note that we can write
    \begin{equation}
    Q_H(\tX)_{i_{a},i_b} = \sum_{\substack{j_1\ne\cdots \ne j_m \\ j_a=i_a, j_b=i_b}}\prod_{(r,s)\in E(\tilde H)} \tX_{j_r,j_s},
    \end{equation}
    where we replace the product over $E(H)$ with the product over $E(\tilde H)$, due to the assumption that $\tX$ is binary. Further, note that the summand $\prod_{(r,s)\in E} \tX_{j_r,j_s}$ is either $0$ or $1$ for each setting of $(j_1, \ldots, j_m)$.
    We show that the number of nonzero terms of the sum in $Q_H(\tX)_{i_a, i_b}$ is equal to $|\Aut(\tilde H)| \cdot \scount(\tilde H, \tX, (i_a, i_b))$. 
    
     Let $(j_1, \ldots, j_m)$ correspond to a nonzero term in the sum. Define $\varphi: V(H) \to V(\tX)$ by $\varphi(i) = j_i$. Note that $\varphi$ is injective since $j_1 \neq \ldots \neq j_m$. Moreover, $\varphi(a) = j_a = i_a$ and $\varphi(b) = j_b = i_b$. Now, define
    \begin{equation}
        V_\varphi = \{\varphi(i) : i \in V(H) \}, \quad E_\varphi = \{(\varphi(r), \varphi(s)) : (r,s) \in E(\tilde H) \}.
    \end{equation}
    This is the vertex set and edge set of a subgraph in $\tX$, because $\prod_{(r,s) \in E(\tilde H)} \tX_{j_r, j_s} = 1$ implies that $(\varphi(r), \varphi(s)) \in E(\tX)$ for all $(r,s) \in E(\tilde H)$. Thus, $\varphi$ corresponds to an isomorphism between $\tilde H$ and a subgraph of $\tX$.

    Suppose $(j_1, \ldots, j_m) \neq (\tilde j_1, \ldots, \tilde j_m)$ are both indices corresponding to nonzero summands, with corresponding subgraph isomorphisms $\varphi$ and $\tilde \varphi$. Note that $\varphi \neq \tilde \varphi$, so each nonzero summand corresponds to a unique subgraph isomorphism $\varphi$. Hence, it suffices to show that the number of subgraph isomorphisms is $|\Aut(\tilde H)| \cdot \scount(\tilde H, \tX, (i_a, i_b)) $.
    
    For each $l \in \{1, \ldots, \scount(\tilde H, \tX, (i_a, i_b)) \}$, let $G_l = (V_l, E_l)$ be a subgraph of $\tX$ that is isomorphic to $\tilde H$. Then choose a subgraph isomorphism $\varphi_l: V(H) \to V(\tX)$ associated to $G_l$. For this $\varphi_l$, as in our argument above we know that $\varphi_l \circ \sigma: V(H) \to V(\tX)$ is a subgraph isomorphism for every $\sigma \in \Aut(\tilde H)$. Thus, there are at least $|\Aut(\tilde H)| \cdot \scount(\tilde H, \tX, (i_a, i_b))$ subgraph isomorphisms $\varphi_l \circ \sigma$.

    To show that there are at most $|\Aut(\tilde H)| \cdot \scount(\tilde H, \tX, (i_a, i_b))$ subgraph isomorphisms, assume for sake of contradiction that $\psi : V(H) \to V(\tX)$ is a subgraph isomorphism that is not of the form $\varphi_l \circ \sigma$ above. Denote the vertex set and edge set of the associated subgraph as $V_\psi$ and $E_\psi$, respectively. If $(V_\psi, E_\psi) \neq (V_l, E_l)$ for each $l$, then we have the existence of $\scount(\tilde H, \tX, (i_a, i_b)) + 1$ subgraphs in $\tX$ isomorphic to $\tilde H$, which contradicts the definition of $\scount$. Thus, $(V_\psi, E_\psi) = (V_l, E_l)$ for some $l$. As $\varphi_l$ and $\psi$ are both bijective from $V(\tilde H) \to V_\psi$, there is a unique bijection $\sigma: V(\tilde H) \to V(\tilde H)$ such that $\varphi_l = \psi \circ \sigma$.
    
    We will show that $\sigma \in \Aut(\tilde H)$, which contradicts our definition of $\psi$. Note that $\sigma(a) = a$ and $\sigma(b) = b$, because $\varphi_l(a) = \psi(a) = i_a$ and $\varphi_l(b) = \psi(b) = i_b$. Now, suppose $(r,s) \in E(\tilde H)$. Since $(\psi \circ \sigma(r), \psi \circ \sigma(s)) = (\varphi(r), \varphi(s)) \in E_{\varphi_l} = E_\psi$, we know that $(\sigma(r), \sigma(s)) \in E(\tilde H)$ as $\psi$ is a subgraph isomorphism. On the other hand, if $(\sigma(r), \sigma(s)) \in E(\tilde H)$, then $(\varphi_l(r), \varphi_l(s)) = (\psi \circ \sigma(r), \psi \circ \sigma(s)) \in E_\psi$ since $\psi$ is a subgraph isomorphism, so that $(r,s) \in E(\tilde H)$ because $\varphi_l$ is a subgraph isomorphism. Thus, $\sigma \in \Aut(\tilde H)$, and we are done.
    
\end{proof}

\section{k-WL Equivalence}

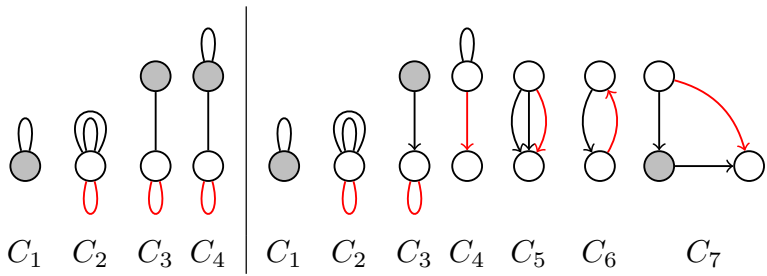
\begin{figure}[h!]
    \centering
    \resizebox{.6\columnwidth}{!}{
    \begin{tabular}{@{\hskip0pt}c@{\hskip2pt}c@{\hskip2pt}c@{\hskip5pt}c|c@{\hskip2pt}c@{\hskip2pt}c@{\hskip5pt}c@{\hskip5pt}c@{\hskip5pt}c@{\hskip5pt}c@{\hskip0pt}}
   \begin{tikzpicture}[main/.style = {draw, circle},baseline=5pt,semithick,every loop/.style={}]
\node[main] (1) [fill=lightgray] {};
\path 
(1) edge [-, loop above] node {} (1);
\end{tikzpicture}
&  
 \begin{tikzpicture}[main/.style = {draw, circle},baseline=5pt,semithick,every loop/.style={}]
\node[main] (1) [] {};
\path 
(1) edge [-, loop below, color=red] node {} (1)
(1) edge [-, loop above] node {} (1)
(1) edge [out=120, in=60, loop, min distance = 20pt] (1);
\end{tikzpicture}
& 
\begin{tikzpicture}[main/.style = {draw, circle},baseline=5pt,semithick,every loop/.style={}]
\node[main] (2) [] {};
\node[main] (1) [above of=2, fill=lightgray] {};
\path 
(1) edge [-] node {} (2)
(2) edge [-, loop below, color=red] node {} (2);
\end{tikzpicture} & 
     \begin{tikzpicture}[main/.style = {draw, circle},baseline=5pt,semithick,every loop/.style={}]
\node[main] (2) [] {};
\node[main] (1) [above of=2, fill=lightgray] {};
\path 
(1) edge [loop above,-] node {} (1)
(1) edge [-] node {} (2)
(2) edge [-, loop below, color=red] node {} (2);
\end{tikzpicture} 
&

\begin{tikzpicture}[main/.style = {draw, circle},baseline=5pt,semithick,every loop/.style={}]
\node[main] (1) [fill=lightgray] {};
\path 
(1) edge [-, loop above] node {} (1);
\end{tikzpicture}
&  
 \begin{tikzpicture}[main/.style = {draw, circle},baseline=5pt,semithick,every loop/.style={}]
\node[main] (1) [] {};
\path 
(1) edge [-, loop below, color=red] node {} (1)
(1) edge [-, loop above] node {} (1)
(1) edge [out=120, in=60, loop, min distance = 20pt] (1);
\end{tikzpicture}
& 
\begin{tikzpicture}[main/.style = {draw, circle},baseline=5pt,semithick,every loop/.style={}]
\node[main] (2) [] {};
\node[main] (1) [above of=2, fill=lightgray] {};
\path 
(1) edge [->] node {} (2)
(2) edge [-, loop below, color=red] node {} (2);
\end{tikzpicture} 
&

\begin{tikzpicture}[main/.style = {draw, circle},baseline=5pt,semithick,every loop/.style={}]
\node[main] (2) [] {};
\node[main] (1) [above of=2] {};
\path 
(1) edge [-, loop above] node {} (1)
(1) edge [->, color=red] node {} (2);
\end{tikzpicture}
& 
\begin{tikzpicture}[main/.style = {draw, circle},baseline=5pt,semithick,every loop/.style={}]
\node[main] (2) [] {};
\node[main] (1) [above of=2] {};
\path 
(1) edge [->] node {} (2)
(1) edge [->, bend right] node {} (2)
(1) edge [->, bend left, color=red] node {} (2);
\end{tikzpicture}
& 
\begin{tikzpicture}[main/.style = {draw, circle},baseline=5pt,semithick,every loop/.style={}]
\node[main] (2) [] {};
\node[main] (1) [above of=2] {};
\path 
(1) edge [->, bend right] node {} (2)
(1) edge [<-, bend left, color=red] node {} (2);
\end{tikzpicture}
& 
\begin{tikzpicture}[main/.style = {draw, circle},baseline=5pt,semithick,every loop/.style={}]
\node[main] (2) [fill=lightgray] {};
\node[main] (1) [above of=2] {};
\node[main] (3) [right of=2] {};
\path 
(1) edge [->] node {} (2)
(2) edge [->] node {} (3)
(1) edge [->, bend left, color=red] node {} (3);
\end{tikzpicture}\\
$C_1$ & $C_2$ & $C_3$ & $C_4$ & $C_1$ & $C_2$ & $C_3$ & $C_4$ & $C_5$ & $C_6$ & $C_7$
\end{tabular} }
    \caption{For convenience, we redraw our Prototypical graph models here. We show these the node-based model (left) is equivalent to 1-WL on simple graphs, and the edge-based model (right) is equivalent to 2-FWL / 3-WL on simple graphs.}
\end{figure}

In this section, we demonstrate that our studied Prototypical graph models achieve 1-WL and 3-WL/2-FWL expressive power, thus showing that our framework can be used to design and analyze k-WL style models. The key connection comes from a result of ~\citet{dvovrak2010recognizing, dell2018lov}, which states that $k$-FWL indistinguishability is equivalent to $\hom(H, \tX)$ indistinguishability for all $H$ of tree-width at most $k$. 
\begin{lemma}[\citet{dvovrak2010recognizing, dell2018lov}]
    Two simple graphs $\tX^{(1)}$ and $\tX^{(2)}$ are $k$-FWL distinguishable if and only if there is a graph $H$ of tree-width at most $k$ such that $\hom(H, \tX^{(1)}) \neq \hom(H, \tX^{(2)})$
\end{lemma}
Recall from Appendix~\ref{appendix:homomorphism} that $\hom(H, \tX)$ is equal to the evaluation of the invariant polynomial $P_H(\tX)$ when $\tX \in \{0, 1\}^{n \times n}$. Thus, an Prototypical graph model can compute $\hom(H, \tX)$ if it can contract $H$ into the trivial graph (that has zero nodes and zero edges) using contractions from its bank.
\begin{proposition}
    The Prototypical node-based graph model can distinguish any two simple graphs if and only if 1-WL can.
\end{proposition}
\begin{proof}
$(\impliedby)$ Suppose 1-WL can distinguish two simple graphs $\tX^{(1)}$ and $\tX^{(2)}$. Then there is a graph $H$ of tree-width 1 such that $\hom(H, \tX^{(1)}) \neq \hom(H, \tX^{(2)})$. We will show that the node-based Prototypical model can contract $H$ to the trivial graph. As $H$ has tree-width 1, it is a forest or a tree. We can assume it is a tree as we can contract each tree connected component one by one if it is a forest. We will show that the node-based model can contract any tree $T$ with or without self-loops.

Suppose $T$ consists of one node. If $T$ has no self-loops, it is the constant polynomial, which is a trivial case. Otherwise, the node has at least one self-loop, and $C_1$ and $C_2$ can of course contract it the trivial graph.

Now, suppose $T$ consists of $m \geq 2$ nodes. Then there is a leaf (a node that has degree 1 when we ignore self-loops). If this leaf does not have a self-loop, we can contract it to remove this node using $C_3$. Otherwise, we can use $C_2$ to remove any multiple self-loops (if needed), and then use $C_4$ to remove the node once we are left with one self-loop. This gives a tree $T'$ with at least one self-loop of $m-1$ nodes, so $T'$ can be contracted by induction.

$(\implies)$ Suppose 1-WL cannot distinguish the two simple graphs $\tX^{(1)}$ and $\tX^{(2)}$, so $\hom(H, \tX^{(1)}) = \hom(H, \tX^{(2)})$ for all $H$ of tree-width 1. We show that the node-based contraction bank cannot contract any $\hom(H, \tX)$ for $H$ of tree-width greater than 1.

Suppose $H$ has tree-width greater than 1, so it cannot be a forest or a tree. Hence, $H$ must have a cycle. However, the node-based model cannot contract any graph that has a cycle; this is because if it could, then the first node of the cycle that is contracted would have had at least two different neighbors, but such a node cannot be contracted by $C_1, C_2, C_3$ or $C_4$. Hence, the node-based graph model cannot distinguish $\tX^{(1)}$ and $\tX^{(2)}$.
\end{proof}

\begin{proposition}
    The Prototypical edge-based graph model can distinguish any two simple graphs if and only if 2-FWL / 3-WL can.
\end{proposition}
\begin{proof}
$(\impliedby)$  Suppose 2-FWL can distinguish the two simple graphs $\tX^{(1)}$ and $\tX^{(2)}$, so there is a graph $H$ of tree-width 2 such that $\hom(H, \tX^{(1)}) \neq \hom(H, \tX^{(2)})$. We will show that the edge-based model can contract $H$. We can assume $H$ is connected because otherwise we can separately contract each connected component. Moreover, we can assume that $H$ has no self-loops or multiple edges. This is because if a node has a self-loop and it has no neighbors, then $C_2$ and $C_1$ can remove the node, and if it has neighbors then $C_4$ can contract the self-loop and add an edge to a neighbor, thus forming a multiple edge. For multiple edges, $C_6$ can align the direction of the edges between if necessary, and then $C_5$ can remove the multiple edges.

Since $H$ has tree-width 2, we know it is a partial 2-tree. Thus, there is an ordering of the $m$ vertices of $H$, say $v_1, \ldots, v_m$, such that when deleting each vertex and all incident edges in turn, we only ever delete vertices of degree at most 2. We show that our edge-based Prototypical model can contract edges in this order by induction. For $i \in [m]$, suppose we have deleted nodes $v_1, \ldots, v_{i-1}$ (in the base case $i=1$ we have not deleted any nodes), we are deleting $v_i$, the current graph has no multiple edges, and the any self-loops in the current graph belong to nodes with no neighbors
\begin{itemize}
    \item  If node $v_i$ has degree zero, then it can be contracted to a trivial graph by $C_1$ and $C_2$.
    \item If node $v_i$ has degree one, then we use $C_3$ to contract it, thus adding a self-loop to its neighbor. If its neighbor then has degree 0, then we do not need to remove self-loops for our induction. Otherwise, if its neighbor has nonzero degree, then we use $C_4$, $C_6$ (if necessary), and $C_5$ to remove the self-loop of the neighbor, and remove any multiple edges.

    \item If node $v_i$ has degree two, then we use $C_6$ if necessary, then use $C_7$ to contract $v_i$, and use $C_6$ and $C_5$ to remove any multiple edges formed.
\end{itemize}
After any of these operations, we have deleted the node $v_i$, and maintained the assumptions of our induction. In particular, note that the degree of a node (ignoring self-loops) never increases, so we indeed only ever contract nodes of degree at most 2. Thus, the Prototypical edge model is capable of contracting $H$ to a trivial graph.

$(\implies)$  Consider a pair of graphs that are not distinguishable by 3-WL. Let $k > 3$ be the smallest integer such that $k$-WL distinguishes this pair. This pair of graphs differs in the number of homomorphisms for a subgraph that has at least tree-width $k-1 \geq 3$ \cite{grohe2021homomorphism} (e.g. see figure 2 of \cite{bouritsas2022improving} for explicit example). Series of eliminations of nodes of a subgraph form a tree decomposition of the subgraph via its chordal completion as described below. Subgraphs of tree-width $\geq 3$ have at least one bag of 4 or more nodes in their tree decomposition. Since the Prototypical edge-based graph model can only contract with bags of size 3 or fewer, such a homomorphism count cannot be performed using the contractions.

To show that any set of contractions forms a valid tree decomposition of the graph, consider the chordal completion of the graph formed by the eliminations. This chordal completion consists of the original graph with all edges added between nodes that were part of a contraction which eliminated any other node. E.g., if nodes in $\{a,b,c\}$ were part of a contraction eliminating node $a$, then all edges between the nodes in $\{a,b,c\}$ are added to its chordal completion. Note that this chordal completion follows naturally in any contraction as once a node is eliminated in a contraction, an edge must be made between its remaining neighbors to store the output of the contraction.

Any chordal completion of a subgraph has the property that via the same elimination order of its construction, no more edges are added. I.e., the subgraph constructed by each eliminated node and its neighbors forms a clique. This elimination ordering forms a tree decomposition with bags consisting of the cliques in the elimination ordering. Thus, the tree-width of a graph is upper bounded by the size of the maximal clique in its chordal completion minus one \cite{graphtheorytextbook}. Since the Prototypical edge based model can only contract up to 3 nodes at once, cliques of size at most $3$ can be constructed in the chordal completion and the tree-width of such a tree decomposition is at most $2$.
\end{proof}

Another way to approach this result is through the relationship between graphs of tree-width 2 and series-parallel graphs. Any biconnected graph $H$ of tree-width 2 is a series-parallel graph (more generally, a graph has tree-width 2 if and only if all biconnected components are series-parallel)~\citep{bodlaender1998partial}. Suppose $H$ is series-parallel. Then it is known that it can be contracted to a single edge by two operations~\citep{duffin1965topology}:
\begin{itemize}
    \item[(op1)] Delete a node of exactly degree 2, and connecting its two neighbors.
    \item[(op2)] If there are two edges between the same two nodes, delete one of the edges.
\end{itemize}
The edge-based Prototypical model can implement these two operations, so it can contract $H$ into a single edge. The first operation (op1) is $C_7$ (matrix multiplication), possibly with a $C_6$ (matrix transpose) beforehand to align the directions of the edges. The second operation (op2) is $C_5$ (replace two parallel edges with one edge), again possibly with a $C_6$ (matrix transpose) beforehand to align the directions of the edges. After contraction to a single edge, the remaining operations $C_1, C_2, C_3$ can be used to contract $H$ to the trivial graph.

If $H$ is not biconnected, then we can use the block-cut tree of $H$ to get the biconnected components. Then we can contract each biconnected component that is a leaf of the block-cut-tree (as it is a series parallel graph) in a way such that we add a self-loop to the cut vertex it is connected to. After pruning cut vertices appropriately, we can continue this process until reaching the trivial graph.

\section{Counting of Invariant Polynomials of Symmetric Group}
\label{app:counting_section}
Various methods exist to count the number of invariant polynomials of the Symmetric group (in our case, also isomorphic to multigraphs) of a given form \cite{harary2014graphical,polya1937kombinatorische,MolienTheorem,thiery2000algebraic,bedratyuk2015new}. Here, we follow a standard strategy to count the number of invariant polynomials by summing over partitions corresponding to cycle indices of the permutation group. Given $S_n$ as the symmetric group on $n$ elements, let $S_n^{(k)}$ be the symmetric group acting on the representation $X \in \left(\mathbb{R}^{n}\right)^{\otimes k}$. Let $P_n$ denote the set of integer partitions of $n$ where each partition $\vm \in P_n$ is a length $n$ vector whose $j$-th element is the number of elements of size $j$ in the given partition. For example, for the partition of $4$ into $(1,1,2)$, the corresponding value of $\vm=[2,1,0,0] $. Let $\vs_i$ for $i \in \mathbb{N}$ be arbitrary variables for now (their meaning will become clear later). We define the cycle index $Z(S_n^{(k)})[\vs_i]$ of $S_n^{(k)}$ as a power series in variables $\vs_i$ for $i \in \mathbb{N}$ as  
\begin{equation}
    Z(S_n^{(k)})[\vs_i] =   \sum_{\vm \in P_n} \frac{1}{\prod_{t=1}^n t^{\vm_t} \vm_t! } \prod_{i = 1}^k \prod_{j_i=1}^n \vs_{\operatorname{lcm}(j_1, \dots, j_k)}^{\prod_{t=1}^k\vm_{j_t} j_t / \operatorname{lcm}(j_1, \dots, j_k) },
\end{equation}
where $\operatorname{lcm}(\cdot)$ is the least common multiple of the arguments.

As an example, we have
\begin{equation}
\label{eq_appZ3_2}
    Z(S_3^{(2)})[\vs_i] =  \frac{1}{6} \vs_1^9 + \frac{1}{2} \vs_1\vs_2^4 + \frac{1}{3} \vs_3^3.
\end{equation}

For the equivariant case, we need a weighted cycle index which we denote as $Z_W(S_n^{(k)} \times S_n^{(d)})[\vs_i]$ which can be calculated as
\begin{equation}
\begin{split}
    Z_W(S_n^{(k)}\times S_n^{(d)})[\vs_i] =   \sum_{\vm \in P_n} \frac{\vm_1^d}{\prod_{t=1}^n t^{\vm_t} \vm_t! } &\prod_{i = 1}^k \prod_{j_i=1}^n \vs_{\operatorname{lcm}(j_1, \dots, j_k)}^{\prod_{t=1}^k\vm_{j_t} j_t / \operatorname{lcm}(j_1, \dots, j_k) },
\end{split}
\end{equation}


From here, we can generate the Molien series which counts the number of invariant/equivariant polynomials on $S_n^{(k)}$.

\begin{theorem}
\label{thm:counting_proof}
The Molien series $M_{S_n^{(k)}}(x)$ for invariant polynomials on $S_n^{(k)}$ is generated by
\begin{equation}
    \begin{split}
        \text{ invariant: }& M_{S_n^{(k)}}(x) = Z(S_n^{(k)})[\vs_i = 1 + x^i + x^{2i} + \cdots],
    \end{split}
\end{equation}
and more generally, the Molien series for the polynomials which are equivariant to $S_n^{(d)}$ outputs and $S_n^{(k)}$ inputs is generated by:
\begin{equation}
    \begin{split}
        \text{ equivariant: }& M_{S_n^{(k)} \times S_n^{(d)} }(x) = Z_W(S_n^{(k)}\times S_n^{(d)})[\vs_i = 1 + x^i + x^{2i} + \cdots].
    \end{split}
\end{equation}
\end{theorem}
\begin{proof}
    We enumerate the Molien series of order $k$ invariant polynomials $R_k^G$ in the invariant ring of polynomials $R^G$ using Molien's formula \cite{derksen2015computational,MolienTheorem}. A similar proof can be obtained via the Pólya enumeration theorem \cite{tucker1994applied,polya1937kombinatorische}. 
    
    Given a representation $\rho:G \to \mathrm{GL}\left(V \right)$, Molien's formula states that
    \begin{equation}
        M_G(x) = \sum_{k} \operatorname{dim}(R_k^G) x^k = |G|^{-1} \sum_{g \in G} \frac{1}{\operatorname{det}\left( \mI - x \rho(g) \right)}.
    \end{equation}
    First, let us consider the invariant setting for $S_n^{(k)}$ -- the symmetric group with representation acting on the vector space $X \in \left(\mathbb{R}^{n}\right)^{\otimes k}$. Eigenvalues of a permutation matrix depend only on the cycle index of the permutation. Therefore, we decompose the sum of group elements in the symmetric group by their cycle indices
    \begin{equation}
    \label{eq:Molien_Snk_raw}
        M_{S_n^{(k)}}(x) = \sum_{\vm \in P_n} \frac{1}{\prod_{t=1}^n t^{\vm_t} \vm_t! } \frac{1}{\operatorname{det}\left( \mI - x \rho(g_{\vm}^{(k)}) \right)},
    \end{equation}
    where $\rho(g_{\vm^{(k)}})$ is any permutation with cycle index $\vm_t$ for the representation on $S_n^{(k)}$. For $S_n^{(1)}$, the representation of the permutation group is the standard representation corresponding to permutations of indices of the vector space. For cycle index $\vm$, the representation $\rho(g_{\vm^{(1)}})$ has $\vm_j$ eigenvalues equal to the $j$ different powers of the $j$-th root of unity. Denoting $\vs_i = 1 + x^i + x^{2i} + \cdots$, this then results in the following:
    \begin{equation}
        \frac{1}{\operatorname{det}\left( \mI - x \rho(g_{\vm}^{(1)}) \right)} = \prod_{j_i=1}^n \vs_{j_i}^{\vm_{j_i}}.
    \end{equation}
    $\rho(g_{\vm^{(k)}})$ acts as a k-fold tensor product of the representation $\rho(g_{\vm^{(1)}})$, i.e. $\rho(g_{\vm^{(k)}}) = [\rho(g_{\vm^{(1)}})]^{(\otimes k)}$. Therefore, to generalize the above formula to higher order $k$, we enumerate the possible eigenvalues of the tensor product of $\rho(g_{\vm^{(1)}}$. Since the eigenvalues of a tensor product of operators are simply the product of the eigenvalues of the elements of the tensor product, we can perform this enumeration over products of the operators. 
    
    As we noted before, for cycle index $\vm$, the representation $\rho(g_{\vm^{(1)}})$ has $\vm_j$ eigenvalues equal to the $j$ different powers of the $j$-th root of unity. Given the product of two elements of this cycle index, we now consider the product of eigenvalues equal to the different powers of the $j$-th and $k$-th roots of unity. This results in all the eigenvalues which are products of the $\operatorname{lcm}(j,k)$-roots of unity where $\operatorname{lcm}(\cdot)$ denotes the least common multiple. Given $\vm_j$ and $\vm_k$ cycles of $j$-th and $k$-th order respectively, powers of the $\operatorname{lcm}(j,k)$-roots of unity will appear a total of $\vm_j \vm_k / \operatorname{lcm}(j,k)$ number of times. Generalizing this to products of more than two eigenvalues, we arrive at
    \begin{equation}
        \frac{1}{\operatorname{det}\left( \mI - x \rho(g_{\vm}^{(k)}) \right)} = \prod_{i = 1}^k \prod_{j_i=1}^n \vs_{\operatorname{lcm}(j_1, \dots, j_k)}^{\prod_{t=1}^k\vm_{j_t} j_t / \operatorname{lcm}(j_1, \dots, j_k) }.
    \end{equation}
    Plugging the above into \Cref{eq:Molien_Snk_raw}, we obtain the desired result. 

    For the equivariant setting, we use the equivariant form of Molien's formula. For maps from a vector space $V$ to another vector space $W$, we consider a representation $\rho:G \to \mathrm{GL}\left(V \right)$ acting on the input space and a representation $\sigma:G \to \mathrm{GL}\left(W \right)$ acting on the output space. Here, Molien's formula takes the form \cite{derksen2015computational,antoneli2008invariants}
    \begin{equation}
        M_G(x) = \sum_{k} \operatorname{dim}(R_k^G) x^k = |G|^{-1} \sum_{g \in G} \frac{\Tr(\sigma(g)^{-1})}{\operatorname{det}\left( \mI - x \rho(g) \right)}.
    \end{equation}
    The above can be shown by noting that the module of equivariant polynomials corresponds to $(R[V] \otimes W)^G$ where $R[V]$ is the ring of polynomials on the vector space $V$. The representation of $S_n^{(d)}$ corresponds to the $d$-fold tensor product of the standard representation of the symmetric group. For a cycle index $\vm$ this representation has one nonzero entries on the diagonal for each cycle of size $1$. Therefore, there are $\vm_1^d$ total nonzero entries each equal to one. Plugging this in, we arrive at the final solution.
\end{proof}

As an example, returning to \Cref{eq_appZ3_2}, we have for invariant polynomials on $S_3^{(2)}$:
\begin{equation}
\begin{split}
    M_{S_3^{(2)}}(x) &=  \frac{1}{6} \bigl[ (1+x+x^2+\cdots)^9 + 3(1+x+x^2+\cdots)(1+x^2+x^4+\cdots)^4 \\
    &\;\;\;\;\;\;\;\;\; + 2 (1+x^3+x^6+\cdots)^3 \bigr] \\
    &= 1 + 2x + 10x^2 + \cdots
\end{split}
\end{equation}

\Cref{thm:counting_proof} quantifies the Molien series for polynomials on $n$ nodes. To obtain the asymptotic limit $M_{S_\infty^{(2)}}(x)$ for invariant polynomials on graphs of arbitrary size, we note that $M_{S_n^{(2)}}(x)$ and $M_{S_\infty^{(2)}}(x)$ agree in powers $x^c$ up to $c=\floor{n/2}$. Therefore, to generate the asymptotic series up to power $n$, it suffices to calculate the corresponding Molien series for $M_{S_{2n}^{(2)}}(x)$. A similar logic can be applied for equivariant polynomials which also have a ``red" edge as well as described in the main text.

\begin{corollary}
    The number of invariant polynomials quantified in the Molien series $M_{S_\infty^{(2)}}(x)$ for the asymptotic limit of $n \to \infty$ nodes agrees with $M_{S_n^{(2)}}(x)$ for $n$ up to the first $\floor{n/2}$ degrees. Similarly, the number of equivariant polynomials quantified in the Molien series $M_{S_\infty^{(2)}}(x)$ for the asymptotic limit of $n \to \infty$ nodes agrees with $M_{S_n^{(2)}}(x)$ for $n$ up to the first $\floor{n/2}-1$ degrees.
\end{corollary}

\paragraph{Counts of invariant polynomials.}
The Molien series of the number of invariant polynomials on graphs, i.e., $\mathbb{R}^{n^2} \to \mathbb{R}$ for sufficiently large $n$, begins with
\begin{equation}
\begin{split}
    &1, 2, 11, 52, 296, 1724, 11060, 74527, 533046, 3999187, \\
    &31412182, 257150093, 2188063401, 19299062896, \\
    &176059781439, 1657961491087, \dots 
\end{split}
\end{equation}

\paragraph{Counts of equivariant polynomials.}
The Molien series of the number of equivariant polynomials on graphs, i.e., $\mathbb{R}^{n^2} \to \mathbb{R}^{n^2}$ for sufficiently large $n$, begins with
\begin{equation}
    \begin{split}
        &2, 15, 117, 877, 6719, 52505, 422824, 3508753, 30036833, \\
        &265100322, 2410638644, 22563597944, 217175819474, \\
        &2147355853088, 21790101729085, 226707665717377, \dots
    \end{split}
\end{equation}

For the setting of the standard representation of the symmetric group on nodes ($S_n^{(1)}$ in our notation), the above recovers the generating series for partitions for which there exist efficiently calculable recurrences via the pentagonal number theorem \cite{hardy1979introduction}. We do not know of a similarly more direct way to compute the Molien series above in the general case.

\end{document}